\documentclass[12pt]{article}

\usepackage{fullpage}
\usepackage{amsfonts,amsmath,epsf,epsfig,bbm}
\usepackage{amsthm}

\usepackage{graphicx}

\title{MAP- and MLE-Based Teaching} 

\author{Hans Ulrich Simon and Jan Arne Telle}

\newtheorem{theorem}{Theorem}[section]
\newtheorem{definition}[theorem]{Definition}
\newtheorem{lemma}[theorem]{Lemma}
\newtheorem{corollary}[theorem]{Corollary}
\newtheorem{remark}[theorem]{Remark}

\newtheorem{example}[theorem]{Example}

\newcommand{\VCD}{\mathrm{VCdim}}
\newcommand{\AN}{\mathrm{AN}}
\newcommand{\sm}{\setminus}
\newcommand{\eset}{\emptyset}
\newcommand{\cO}{{\mathcal O}}
\newcommand{\cZ}{{\mathcal Z}}
\newcommand{\ol}{\overline}
\newcommand{\seq}{\subseteq}
\newcommand{\ra}{\rightarrow}
\newcommand{\dund}{\Leftrightarrow}
\newcommand{\impl}{\Rightarrow}
\newcommand{\da}{\downarrow}
\newcommand{\ua}{\uparrow}
\newcommand{\ve}{\varepsilon}

\let\vec\mathbf

\DeclareMathOperator{\arg!max}{arg!max}
\DeclareMathOperator{\MAPTD}{MAP\mbox{-}TD}
\DeclareMathOperator{\MLETD}{MLE\mbox{-}TD}
\DeclareMathOperator{\SMN}{SMN}
\DeclareMathOperator{\NCTD}{NC\mbox{-}TD}

\begin{document}

\maketitle

\begin{abstract}
Imagine a learner $L$ who tries to infer a hidden concept from a collection
of observations. Building on the work~\cite{IJCAI} of Ferri et al.,
we assume the learner to be parameterized by priors $P(c)$ and by
$c$-conditional likelihoods $P(z|c)$ where $c$ ranges over all concepts
in a given class $C$ and $z$ ranges over all observations in an observation
set $Z$. $L$ is called a \emph{MAP-learner} (resp.~an \emph{MLE-learner})
if it thinks of a collection $S$ of observations as a random sample and
returns the concept with the maximum a-posteriori probability
(resp.~the concept which maximizes the $c$-conditional
likelihood of $S$). Depending on whether $L$ assumes that $S$
is obtained from ordered or unordered sampling resp.~from sampling with
or without replacement, we can distinguish four different sampling modes.
Given a target concept $c^* \in C$, a teacher for a MAP-learner $L$ aims
at finding a smallest collection of observations that causes $L$ to
return $c^*$. This approach leads in a natural manner to various
notions of a MAP- or MLE-teaching dimension of a concept class $C$. 
Our main results are as follows. First, we show that this teaching 
model has some desirable monotonicity properties. Second we clarify 
how the four sampling modes are related to each other. As for the 
(important!) special case, where concepts are subsets of a domain 
and observations are 0,1-labeled examples, we obtain some additional 
results. First of all, we characterize the MAP- and MLE-teaching dimension 
associated with an optimally parameterized MAP-learner graph-theoretically. 
From this central result, some other ones are easy to derive. It is
shown, for instance, that the MLE-teaching dimension is either equal
to the MAP-teaching dimension or exceeds the latter by $1$.
It is shown furthermore that these dimensions can be bounded from above 
by the so-called antichain number, the VC-dimension and related combinatorial 
parameters. Moreover they can be computed in polynomial time. 
\end{abstract}

\section{Introduction} \label{sec:introduction}

In formal models of \emph{machine learning} we have a concept class $C$ of
possible concepts/hy\-po\-the\-ses, an unknown target concept $c^* \in C$
and training data given by correctly labeled random examples.
In formal models of {\em machine teaching} a collection $T(c^*)$ of
labeled examples is instead carefully chosen by a teacher $T$ in
a way that the learner can reconstruct the target concept $c^*$
from $T(c^*)$. In recent years, the field of machine teaching has
seen various applications in fields like explainable AI~\cite{ourECML},
trustworthy AI~\cite{zhu2018overview} and pedagogy~\cite{SHAFTO201455}.

Various models of machine teaching have been proposed, e.g. the
classical teaching model~\cite{SM1991,GK1995}, the optimal
teacher model~\cite{B2008}, recursive
teaching~\cite{ZLHZ2011},  preference-based
teaching~\cite{GRSZ2017}, or no-clash teaching~\cite{KSZ2019,FKSSZ2022}.
These models differ mainly in the restrictions that they impose on the
learner and the teacher in order to avoid unfair collusion or cheating.
The common goal is to keep the size of the largest teaching
set, $\max_{c \in C}|T(c)|$, as small as possible.

There are also other variants using probabilities, from
Muggleton~\cite{Muggleton} where examples are sampled based on
likelihoods for a target concept, to  Shafto et al.~\cite{SHAFTO201455}
who calls this pedagogical sampling and leads into Bayesian
Teaching~\cite{EavesS16,yang2017explainable}, to the Bayesian learners
of Zhu~\cite{NIPS2013_9c01802d} with a proper teacher selecting examples.

In this paper we continue this line of research and consider the
probabilistic model that had been described in the abstract.
This model is inspired by and an extension of the model that was
introduced in~\cite{IJCAI}. As already observed in~\cite{IJCAI},
the condition for collusion-avoidance from~\cite{GM1996}
may here be violated, i.e., the learner may first reconstruct a concept $c_1$
from some given observations but, after having received additional
observations, switch to another concept $c_2$ even if the new
observations have given additional support to $c_1$. As the authors
of~\cite{IJCAI}, we would like to  argue that this should not be
considered as collusion or cheating as long as the parameters assigned
to the learner reflect some factual information about the world.

As already outlined in the abstract, we will distinguish between four
distinct sampling modes: ordered sampling with replacement
($(O,R)$-mode), unordered sampling with replacement ($(\ol{O},R)$-mode),
ordered sampling without replacement ($(O,\ol{R})$-mode) and
unordered sampling without replacement ($(\ol{O},\ol{R})$-mode).
The smallest number $d$ such that every $c^* \in C$ can be taught
to a given MAP-learner $L$ by a collection of at most $d$ observations
is denoted by $\MAPTD_L^{\alpha,\beta}(C)$
where $(\alpha,\beta) \in \{O,\ol{O}\} \times \{R,\ol{R}\}$
indicates the underlying sampling mode.
Then $\MAPTD^{\alpha,\beta}(C) = \min_L \MAPTD_L^{\alpha,\beta}(C)$
is the corresponding parameter with an optimally parameterized
learner $L$. The analogous notation is used for MLE-learners.
Our main results are as follows:
\begin{enumerate}
\item
The MAP-teaching model has two desirable and quite intuitive monotonicity
properties. Loosely speaking, adding new observations (making $Z$ larger)
leads to smaller $\MAPTD$ while adding new concepts (making $C$ larger)
leads to larger $\MAPTD$. See Section~\ref{subsec:monotonicity} for details.
\item
The sampling modes $(O,R)$ and $(\ol{O},R)$ are equivalent.
The sampling modes $(\ol{O},R)$, $(O,\ol{R})$ and $(\ol{O},\ol{R})$
are pairwise incomparable (i.e., which one leads to smaller
values of $\MAPTD_L(C)$ depends on the choice of $C$ and $L$).
Note that incomparability of the modes $(\alpha,\beta)$
and $(\alpha',\beta')$ does not rule out the possibility
that $\MAPTD^{\alpha,\beta}(C) \le \MAPTD^{\alpha',\beta'}(C)$
for each concept class $C$. See Section~\ref{subsec:modes} for 
details.
\item
As for the (important!) special case, where concepts are subsets of a domain 
and observations are 0,1-labeled examples, we obtain some additional
results, the first of which is the central one:
\begin{enumerate}
\item
For a (properly defined) bipartite graph $G(C)^{\alpha,\beta}$
associated with $C$ and $(\alpha,\beta) \neq (O,R)$, one
gets\footnote{$\SMN(G)$ denotes the saturating matching number
of a bipartite graph $G$ (formally defined in Section~\ref{sec:smn})}
\begin{equation} \label{eq:smn-characterization}
\MAPTD^{\alpha,\beta}(C) = \SMN(G(C)^{\alpha,\beta}) \enspace .
\end{equation}
If we replace $G(C)^{\alpha,\beta}$ by a slightly modified
graph, we obtain the corresponding result for $\MLETD$ at the
place of $\MAPTD$.\footnote{Some bounds on $\MLETD$ numbers
in terms of $\SMN$ numbers are already found in~\cite{IJCAI},
but no results that hold with equality (as in~(\ref{eq:smn-characterization}))
are proven there.} Fig.~\ref{fig:smn-characterization} visualizes this result.
See Sections~\ref{sec:smn} and~\ref{subsec:smn} for details.
\item
The MLE-teaching dimension is either equal to the MAP-teaching dimension 
or exceeds the latter by $1$. See Section~\ref{subsec:map-versus-mle}
for details.
\item
The MAP- and the MLE-teaching dimension can be bounded from above 
by the so-called antichain number, the VC-dimension and related 
combinatorial parameters. See Section~\ref{subsec:mletd-upper-bounds} 
for details.
\item
Moreover the MAP- and the MLE-teaching dimension can be computed 
in polynomial time from a natural encoding of the underlying concept class. 
See Section~\ref{subsec:polytime} for details.
\end{enumerate}
\end{enumerate}

\begin{figure}[hbt] \label{fig:smn-characterization}
\begin{center}
\includegraphics[scale=.65]{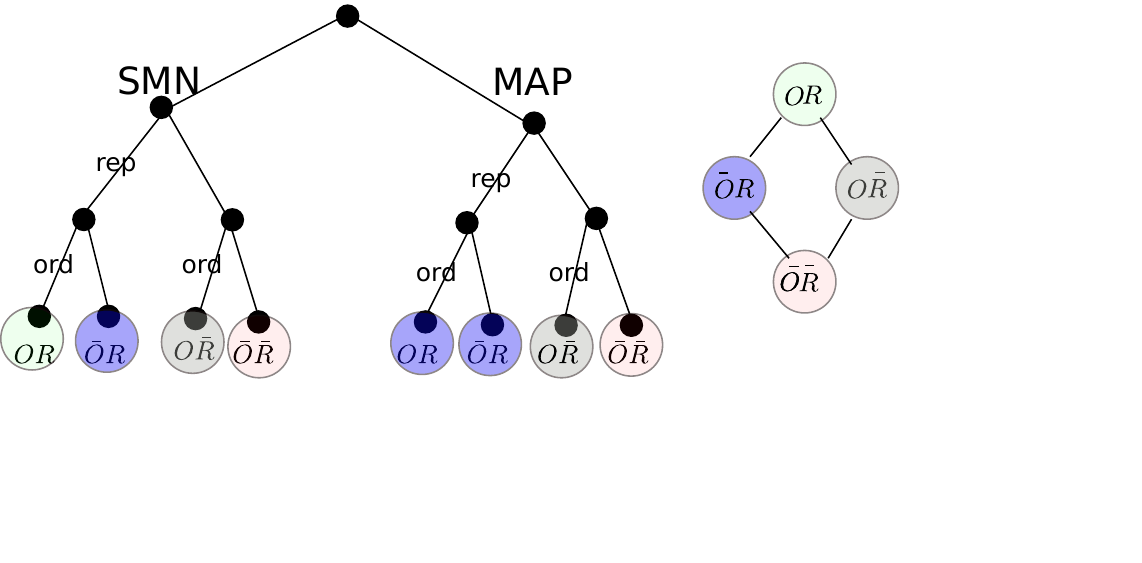}
\end{center}
\vspace{-2cm}
\caption{For any binary concept class $C \seq 2^X$ and $0,1$-labeled
examples as observations, the tree visualizes the identities
in~(\ref{eq:smn-characterization}). Using the same color for the
two leftmost leaves in the MAP-subtree is justified by the equivalence
of the modes $(O,R)$ and $(\ol{O},R)$. A parameter represented by a
leaf in the MAP-subtree has the same value as the parameter represented
by a leaf of the same color in the SMN-subtree. The parameters represented
in the SMN-subtree are ordered as indicated by the rightmost diagram,
with lowest value on top and highest value at bottom. We will see later
that parameters represented in different colors can generally have
different values.}
\end{figure}

\section{Definitions and Notations} \label{sec:definitions}


We first fix some general notation. Afterwards, 
in Sections~\ref{subsec:concept-class},~\ref{subsec:sampling},
and~\ref{subsec:map-mle}, the MAP- and MLE-based teaching model 
is introduced, step-by-step.

\paragraph{Mappings.}

The restriction of a mapping $f:A \ra B$ to a subset $A' \seq A$
will be denoted by $f_{\da A'}$. Suppose that $B$ is a set that is 
equipped with a size function which associates a size $|b|$ with 
each $b \in B$. Then the \emph{order of a mapping $f: A \ra B$} is 
defined as the size of the largest element in the image of $f$, 
i.e., the order of $f$ equals $\max_{a \in A}|f(a)|$.

\paragraph{Graphs and Matchings.}
For a graph $G = (V,E)$ and a set $U \seq V$, we denote by $\Gamma(U)$
the set of vertices which are adjacent to at least one vertex in $U$.
If $G = (V_1,V_2,E)$ is the bipartite graph with vertex sets $V_1$ 
and $V_2$ and with edge set $E \seq V_1 \times V_2$, then $U \seq V_1$
implies (of course) that $\Gamma(U) \seq V_2$. 
A matching $M$ in a bipartite graph $G = (V_1,V_2,E)$ can be viewed
as a (partially defined and injective) function $M: V_1 \ra V_2$
with the property that $(v,M(v)) \in E$ for each $v$ having an
$M$-partner. If $V_1$ is \emph{saturated by $M$}, i.e., every
vertex in $V_1$ has an $M$-partner, then this function is fully defined.

\paragraph{VC-Dimension~\cite{VC1971}.}
Let $C$ be a family of subsets of some ground set $X$. For $c \in C$
and $x \in X$, we also write $c(x)=1$ if $x \in c$ and $c(x)=0$
if $x \notin c$. We say that $S \seq X$ is \emph{shattered by $C$}
if, for every $b:S \ra \{0,1\}$, there is some $c \in C$ that
coincides with $b$ on $S$. The \emph{VC-dimension of $C$} is
defined as $\infty$ if there exist arbitrarily large shattered sets,
and it is defined as the size of a largest shattered set otherwise.

\subsection{Concept Classes} \label{subsec:concept-class}

Let $C$ be a finite set of size at least $2$, let $Z$ be another
non-empty finite set and let $\models$ be a relation on $C \times Z$.
We refer to $C$ as a \emph{concept class} and to $Z$ as a set
of \emph{observations}. If $c \models z$, then we say that
the concept $c$ is \emph{consistent with the observation} $z$.
We say that $c$ is \emph{consistent with a set (resp.~multiset)~$A$
of observations}, which is written as $c \models A$, if $c$ is
consistent with every $z \in A$. The notation $c \models \vec{z}$
with $\vec{z} = (z_1,\ldots,z_n) \in Z^n$ is understood analogously. 
For each $c \in C$, we define
\[ 
Z_c = \{z \in Z: c \models z\} \enspace .
\]

\begin{example}[Positive Examples as Observations]
\label{ex:positive-examples}
Let $Z = X$ be a set of \emph{examples} and let $C$ be a family 
of subsets of $X$. Let the consistency relation be given by
\[
\forall c \in C , x \in X: c \models x \dund x \in c \enspace .
\]
Note that $Z_c = c$ in this setting, i.e., concepts are identified 
with the sets of observations they are consistent with. 
\end{example}

\begin{example}[Labeled Examples as Observations]
\label{ex:labeled-examples}
Let $Z = X \times \{0,1\}$ be a set of \emph{labeled examples} 
and let $C$ be a family of subsets of $X$. Let the consistency 
relation be given by
\begin{equation} \label{eq:labeled-examples}
\forall c \in C , (x,b) \in Z:
c \models (x,b) \dund (b=1 \wedge x \in c) \vee (b=0 \wedge x \notin c)
\enspace .
\end{equation}
Note that $Z_c = \{(x,1): x \in c\} \cup \{(x,0): x \notin c\}$
in this setting. It follows that $|Z_c| = |X|$ for all $c \in C$.
\end{example}
We will occasionally identify a set $c \seq X$ with the 
corresponding $0,1$-valued function so that $c(x)=1$ for $x \in c$
and $c(x) = 0$ for $x \in X \sm c$. The equivalence 
in~(\ref{eq:labeled-examples}) can then be written in the
form $c \models (x,b) \dund b = c(x)$.

\begin{example}[Labeled Examples and Probabilistic Concepts]
\label{ex:probabilistic-concepts}
Let $Z = X \times \{0,1\}$ be again a set of labeled examples
and let $C$ be a family of functions from $X$ to $[0,1]$.
Let the consistency relation be given by
\[ 
\forall c \in C, x \in X:
c \models (x,1) \dund c(x)>0\ \mbox{ and }\  c \models (x,0) \dund c(x)<1
\enspace .
\]
Intuitively we should think of $c(x)$ as the probability that $c$
assigns label $1$ to instance $x$. If all concepts $c \in C$ were
$0,1$-valued, we would again be in the setting of 
Example~\ref{ex:labeled-examples}.
\end{example}

\noindent
Note that within Examples~\ref{ex:positive-examples} 
and~\ref{ex:labeled-examples}, we have that
\begin{equation} \label{eq:extensional-view}
\forall c,c' \in C:\ c \neq c' \impl Z_c \neq Z_{c'} 
\end{equation}
so that each concept $c \in C$ is uniquely determined by the full
set $Z_c$ of observations that $c$ is consistent with. Of course
this will not necessarily be the case if the concepts are probabilistic
as in Example~\ref{ex:probabilistic-concepts}.

\subsection{Variants of Sampling} \label{subsec:sampling}

As formalized in the definitions below, we distinguish between
ordered and unordered sampling and we may have sampling with or
without replacement.

\begin{definition}[Sampling with Replacement] \label{def:sampling-with}
Let $Q = (q(z))_{z \in Z}$ be a \emph{collection of probability 
parameters}, i.e., $q(z) \ge 0$ and $\sum_{z \in Z}q(z) = 1$. 
For $n \ge 0$, we define \emph{$n$-fold (ordered resp.~unordered) 
$Q$-sampling with replacement} as the following random procedure:
\begin{enumerate}
\item 
Choose $z_1,\ldots,z_n$ independently at random according to $Q$.
\item
In case of ordered sampling, return the sequence $(z_1,\ldots,z_n)$
whereas, in case of unordered sampling, return 
the multiset $\{z_1,\ldots,z_n\}$.\footnote{If $n=0$, then the
empty sequence resp.~the empty multiset is returned,}
\end{enumerate}
\end{definition}

Let $\vec{z} = (z_1,\ldots,z_n) \in Z^n$ be a sequence 
that contains $k$ distinct elements, say $z'_1,\ldots,z'_k$,
and let $n_i$ denote the number of occurrences of $z'_i$ in $\vec{z}$. 
Let $A_{\vec{z}} \seq Z$ be the corresponding multiset. The probability 
that $\vec{z}$ (resp.~$A_{\vec{z}}$) is obtained from $n$-fold ordered 
(resp.~unordered) $Q$-sampling with replacement is henceforth denoted 
by $P^{O,R}(\vec{z}|Q)$ (resp.~by $P^{\ol{O},R}(A_{\vec{z}}|Q)$).
With these notations, the following holds:
\begin{equation} \label{eq:R-probs} 
P^{O,R}(\vec{z}|Q) = \prod_{i=1}^{n} q(z_i) = \prod_{i=1}^{k} q(z'_i)^{n_i}
\ \mbox{ and }\ 
P^{\ol{O},R}(A_{\vec{z}}|Q) = 
\frac{n!}{n_1! \ldots n_k!} \cdot \prod_{i=1}^{k}q(z'_i)^{n_i}
\enspace .
\end{equation}

\begin{definition}[Sampling without Replacement] 
\label{def:sampling-without}
Let $Q = (q(z))_{z \in Z}$ be a collection of probability parameters.
Let $N^+(Q)$ be the number of $z \in Z$ such that $q(z)>0$. 
For $0 \le n \le N^+(Q)$, we define \emph{$n$-fold (ordered 
resp.~unordered) $Q$-sampling without replacement} as the
following random procedure:
\begin{enumerate}
\item
Choose $z_1$ at random according to $Q$.
\item 
For $i=2,\ldots,n$ do the following: \\
Choose $z_i \in Z \sm \{z_1,\ldots,z_{i-1}\}$ at random where,
for each $z \in Z \sm \{z_1,\ldots,z_{i-1}\}$, the probability 
for $z_i=z$ equals $\frac{q(z)}{1-(q(z_1) +\ldots+ q(z_{i-1}))}$.\footnote{Note
that the probability parameters for $z \in Z \sm \{z_1,\ldots,z_{i-1}\}$
are the same as before up to normalization.}
\item
In case of ordered sampling, return the sequence $(z_1,\ldots,z_n)$
whereas, in case of unordered sampling, return the set $\{z_1,\ldots,z_n\}$.
\end{enumerate}
\end{definition}

Let $\vec{z} = (z_1,\ldots,z_n) \in Z^n$ be a repetition-free sequence 
and let $A_{\vec{z}} \seq Z$ be the corresponding set. 
For a permutation $\sigma$ of $1,\ldots,n$, 
we define $\vec{z}_\sigma = (z_{\sigma(1)},\ldots,z_{\sigma(n)})$. The
probability that $\vec{z}$ (resp.~$A_{\vec{z}}$) is obtained from
$n$-fold ordered (resp.~unordered) $Q$-sampling without replacement
is henceforth denoted by $P^{O,\ol{R}}(\vec{z}|Q)$ 
(resp.~by $P^{\ol{O},\ol{R}}(A_{\vec{z}}|Q)$). With these notations,
the following holds:
\begin{equation} \label{eq:O-bar-R}
P^{O,\ol{R}}(\vec{z}|Q) = 
\prod_{i=1}^{n}\frac{q(z_i)}{1-(q(z_1) +\ldots+ q(z_{i-1}))} 
\ \mbox{ and }\ 
P^{\ol{O},\ol{R}}(A_{\vec{z}}|Q) =
\sum_\sigma P^{O,\ol{R}}(\vec{z}_\sigma|Q) \enspace ,
\end{equation}
where $\sigma$ ranges over all permutations of $1,\ldots,n$.


\noindent
We introduce the following notation:
\begin{itemize}
\item
$\cZ^{O,R} = Z^*$ denotes the set of sequences over $Z$ 
(including the empty sequence).
\item
$\cZ^{\ol{O},R}$ denotes the set of multisets over $Z$
(including the empty multiset).
\item
$\cZ^{O,\ol{R}}$ denotes the set of repetition-free sequences
over $Z$ (including the empty sequence).
\item
$\cZ^{\ol{O},\ol{R}} = 2^Z$ denotes the powerset of $Z$.
\end{itemize}

The pairs $(\alpha,\beta) \in \{O,\ol{O}\} \times \{R,\ol{R}\}$
are called \emph{sampling modes}. 
We use the symbol $\eset$ not only to denote the empty set
but also to denote the empty multiset or the empty sequence.
If $A$ is a finite set or multiset, then $|A|$ denotes its size
where, in case of a multiset, the multiple occurrences of elements
are taken into account. The length of a finite sequence $\vec{z}$
is denoted by~$|\vec{z}|$. 


\begin{remark}[Trivial Identities] \label{rem:trivial-identities}
Suppose that $Q = (q(z))_{z \in Z}$ is collection of probability 
parameters. Then, for each sampling mode $(\alpha,\beta)$, we have 
that $P^{\alpha,\beta}(\eset | Q) = 1$. Moreover, if all parameters $q(z)$
with $z \in Z$ are strictly positive, then $P^{\ol{O},\ol{R}}(Z | Q) = 1$.
\end{remark}

\noindent
We close this section with a more or less obvious result whose
proof will be given for sake of completeness.

\begin{remark} \label{rem:identity-permutation}
Let $z_1,\ldots,z_n$ be a sequence with pairwise distinct elements from $Z$.
Let $p_1 > p_2 > \ldots p_n$ be a strictly decreasing sequence of
strictly positive parameters such that $\sum_{i=1}^{n}p_i \le 1$.
For each permutation $\sigma$ of $[n]$, consider the parameter 
collection $Q_\sigma = (q_\sigma(z_i))_{i=1,\ldots,n}$ given 
by $q_\sigma(z_i) = p_{\sigma(i)}$. Then the identity permutation
is the unique maximizer of $P^{O,\ol{R}}(z_1,\ldots,z_n | Q_\sigma)$.
\end{remark}

\begin{proof}
According to~(\ref{eq:O-bar-R}), we have
\begin{eqnarray*}
P^{O,\ol{R}}(z_1,\ldots,z_k | Q_\sigma) & = & 
\prod_{i=1}^{n}\frac{q_\sigma(z_i)}{1-(q_\sigma(z_1) +\ldots+ q_\sigma(z_{i-1}))} \\
& = & 
\prod_{i=1}^{n}\frac{p_{\sigma(i)}}{1-(p_{\sigma(1)} +\ldots+ p_{\sigma(i-1)})} =  
\frac{\prod_{i=1}^{n}p_i}{\prod_{i=1}^{n}(1-(p_{\sigma(1)} +\ldots+ p_{\sigma(i-1)})} 
\end{eqnarray*}
The product in the numerator is the same for all permutations $\sigma$.
The following assertions are equivalent:
\begin{enumerate}
\item
$\sigma^*$ is the identity permutation.
\item
The sequence $p_{\sigma^*(1)},\ldots,p_{\sigma^*(n)}$ is strictly decreasing.
\item
For each permutation $\sigma \neq \sigma^*$ and each $i \in [n]$, 
we have that 
\[
p_{\sigma^*(1)} +\ldots+ p_{\sigma^*(i-1)} \ge 
p_{\sigma(1)} +\ldots+ p_{\sigma(i-1)}
\]
and, for at least one $i \in [n]$, this inequality is strict.
\item
The permutation $\sigma^*$ is the unique maximizer
of $P^{O,\ol{R}}(z_1,\ldots,z_k | Q_\sigma)$.
\end{enumerate}
The remark now is immediate from the equivalence of the first and 
the fourth statement.
\end{proof}

\subsection{MAP- and MLE-based Teaching} \label{subsec:map-mle}

An MLE-learner will always choose a hypothesis from a class $C$ 
that maximizes the likelihood of a given set of observations. 
MAP-learners are a bit more general because they additionally 
bring into play priors $(P(c))_{c \in C}$. 
The notion of likelihood depends on how the observations are
randomly sampled. We proceed with the formal definition of MAP-
and MLE-learners and their teachers:

\begin{definition}[MAP- and MLE-Learner] \label{def:mle-learner}
A \emph{MAP-Learner $L$ for $C$} is given by (and henceforth 
identified with) parameters $P(z|c) \ge 0$ 
and $P(c) > 0$ for $z \in Z$ and $c \in C$ such that 
\[
\sum_{c \in C}P(c) = 1\ \mbox{ and }\ 
\forall c \in C: \sum_{z\in Z} P(z|c) = 1 \enspace .
\]
The parameters $P(c)$ are referred to as \emph{priors}.
The parameters $P(z|c)$, referred to as \emph{$c$-conditional 
likelihoods}, must satisfy the following \emph{validity condition}:
\begin{equation} \label{eq:validity}
c \not\models z \impl P(z|c) = 0 \enspace .
\end{equation}
Set $Z_c^+(L) := \{z \in Z: P(z|c) > 0\}$ 
and $N^+(C,L) = \min_{c \in C}|Z_c^+(L)|$.\footnote{Because of the validity 
condition, $Z_c^+(L)$ is a subset of $Z_c = \{z \in Z: c \models z\}$.}
$L$ can be in four different sampling modes (depending on 
the assumed kind of sampling). These modes determine the form
of $L$'s input and the choice of its output as will be detailed
below.
\begin{description}
\item[$(O,R)$-mode:]
For every $n \ge 0$ and every sequence $\vec{a} \in Z^n$, 
we denote by $P^{O,R}(\vec{a}|c)$ the probability that $\vec{a}$ is obtained 
from $n$-fold ordered $P(\cdot|c)$-sampling with replacement. Given 
a sequence $\vec{a} \in \cZ^{O,R}$, $L$ returns 
the concept $\arg!max_{c \in C}\left[P(c) \cdot P^{O,R}(\vec{a}|c)\right]$ 
if it exists, and a question mark otherwise.\footnote{The 
operator $\arg!max_{c \in C}f(c)$ returns the {\bf unique} 
maximizer $c^* \in C$ of $f(c)$ provided that it exists.} 
\item[$(\ol{O},R)$-mode:]
For every $n \ge 0$ and and every multiset $A \seq Z$ of size $n$, 
we denote by $P^{\ol{O},R}(A|c)$ the probability that $A$ is obtained 
from $n$-fold unordered $P(\cdot|c)$-sampling with replacement. Given 
a multiset $A \in \cZ^{\ol{O},R}$, $L$ returns the concept \\
$\arg!max_{c \in C} \left[P(c) \cdot P^{\ol{O},R}(A|c)\right]$ 
if it exists, and a question mark otherwise.
\item[$(O,\ol{R})$-mode:]
For every $0 \le n \le N^+(C,L)$, and every repetition-free 
sequence $\vec{a} \in Z^n$, we denote by $P^{O,\ol{R}}(\vec{a}|c)$) 
the probability that $\vec{a}$ 
is obtained from $n$-fold ordered $P(\cdot|c)$-sampling without 
replacement. Given a repetition-free sequence $\vec{a} \in \cZ^{O,\ol{R}}$
with $|\vec{a}| \le N^+(C,L)$, $L$ returns 
the concept $\arg!max_{c \in C}\left[P(c) \cdot P^{O,\ol{R}}(\vec{a}|c)\right]$ 
if it exists, and a question mark otherwise. If $|\vec{a}| > N^+(C,L)$, then also
a question mark is returned.  
\item[$(\ol{O},\ol{R})$-mode:]
For every $0 \le n \le N^+(C,L)$,
and every set $A \seq Z$ of size $n$, we denote 
by $P^{\ol{O},\ol{R}}(A|c)$ the probability that $A$ is obtained 
from $n$-fold unordered $P(\cdot|c)$-sampling without replacement. 
Given a set $A \in \cZ^{\ol{O},\ol{R}}$ with $|A| \le N^+(C,L)$, $L$ returns 
the concept $\arg!max_{c \in C}\left[P(c) \cdot P^{\ol{O},\ol{R}}(A|c) \right]$ 
if it exists, and a question mark otherwise. If $|A| > N^+(C,L)$, then 
also a question mark is returned.  
\end{description}
An \emph{MLE-learner} is a MAP-learner with uniform priors (so that 
the factor $P(c)$ in the above $arg!max$-expressions can be dropped).
\end{definition}

\begin{definition}[Teacher] \label{def:teacher}
Suppose that $L$ is a MAP-learner for $C$ that is 
in sampling mode $(\alpha,\beta) \in \{O,\ol{O}\} \times \{R,\ol{R}\}$.
A \emph{(successful) teacher} for $L$ is a mapping $T$ which assigns 
to each concept $c_0 \in C$ an input $I = T(c_0)$ for $L$ 
such that $L(I) = c_0$. In other words:
\begin{enumerate}
\item
$I \in \cZ^{\alpha,\beta}$ and, if $\beta = \ol{R}$, then $|I| \le N^+(C,L)$.
\item
$c_0 = \arg!max_{c \in C}\left[P(c) \cdot P^{\alpha,\beta}(I|c)\right]$.
\end{enumerate}
\end{definition}

\noindent
A couple of observations are in place here. 

\begin{remark}
Suppose that $L$ is a MAP-learner for $C$ 
which is in sampling mode $(\alpha,\beta) \in \{O,\ol{O}\} \times \{R,\ol{R}\}$. 
Suppose that $T$ is a teacher for $L$. Then the following holds 
for all $c,c' \in C$:
\begin{equation} \label{eq:general-conditions}
L(T(c)) = c\ ,\ P^{\alpha,\beta}(\eset|c) = 1\ ,\  
P^{\alpha,\beta}(T(c)|c) > 0\ ,\ c \models T(c)\  
\mbox{ and }\ (c \neq c' \impl T(c) \neq T(c')) 
\enspace .
\end{equation}
Moreover, if $L$ is an MLE-learner and $T$ is a teacher for $L$, 
then $T(c) \neq \eset$.
\end{remark}

\begin{proof}
$L(T(c)) = c$ is an immediate consequence of Definitions~\ref{def:mle-learner}
and~\ref{def:teacher}. It now follows that, if  $T(c) = T(c')$, 
then $c = L(T(c)) = L(T(c')) = c'$. In other words, $c \neq c'$
implies that $T(c) \neq T(c')$. $0$-fold sampling conditioned to $c$
yields $\eset$ regardless of how $c$ is chosen. It follows 
that $P^{\alpha,\beta}(\eset|c) = 1$. Assume now for contradiction 
that $P^{\alpha,\beta}(T(c')|c') = 0$. But then $c'$ cannot be the 
unique maximizer of $P^{\alpha,\beta}(T(c')|c)$ in $C$.
This is in contradiction with $L(T(c')) = c'$. Assume for contradiction
that $T(c)$ contains an observation $z \in Z$ such that $c \not\models z$.
It follows that $P^{\alpha,\beta}(T(c)|c) = 0$, which is in contradiction 
with $P^{\alpha,\beta}(T(c)|c) > 0$. Thus $c \models T(c)$. Finally, suppose that 
the priors are uniform, i.e., $P(c) = 1/|C|$ for every $c \in C$. 
Assume for contradiction that $T(c_0) = \eset$ for some $c_0 \in C$.
For every $c \in C$, we have $P(c) \cdot P^{\alpha,\beta}(\eset|c) = P(c) = 1/|C|$.
Hence $c_0$ cannot be unique maximizer of $P(c) \cdot P^{\alpha,\beta}(\eset|c)$
in $C$. This is in contradiction with $L(T(c_0)) = c_0$.
\end{proof}

\noindent
Here is the definition of the parameter that is in the focus 
of our interest:

\begin{definition}[MAP- and MLE-Teaching Dimension] \label{def:mletd}
Suppose that $L$ is a MAP-learner for $C$ who is in sampling 
mode $(\alpha,\beta)$. The \emph{MAP-teaching dimension of $C$ 
given $L$ and $(\alpha,\beta)$}, denoted as $\MAPTD^{\alpha,\beta}_L(C)$, 
is defined as the smallest number $d$ such that there exists a teacher 
of order $d$ for $L$, respectively as $\infty$ if there does not exist
a teacher for $L$. The \emph{MAP-teaching dimension of $C$ with respect
to sampling mode $(\alpha,\beta)$} is then given by 
\[
\MAPTD^{\alpha,\beta}(C) := \min_L \MAPTD^{\alpha,\beta}_L(C) \enspace ,
\]
where $L$ ranges over all MAP-learners for $C$. Similarly, 
the \emph{MLE-teaching dimension of $C$ with respect
to sampling mode $(\alpha,\beta)$} is given by 
$\MLETD^{\alpha,\beta}(C) := \min_L \MAPTD^{\alpha,\beta}_L(C)$
with $L$ ranging over all MLE-learners for $C$.
\end{definition}

The parameter $\MAPTD^{\alpha,\beta}(C)$ equals the number of observations 
needed to teach an optimally parameterized learner. It represents an
information-theoretic barrier that cannot be brocken regardless
of how the learner is parameterized. Of course, this parameter will
generally be smaller than the parameter $\MAPTD_L^{\alpha,\beta}(C)$ 
associated with a ``naturally parameterized'' learner. We close this 
section by mentioning the inequality 
\[ 
\MAPTD^{\alpha,\beta}(C) \le \MLETD^{\alpha,\beta}(C) \enspace ,
\]
which (for trivial reasons) holds for each choice of $C$ 
and $(\alpha,\beta)$.

\section{Basic Results on the MAP-Based Teaching Model}
\label{sec:basic-results}

%

In~\cite{IJCAI}, the authors used a more restrictive condition at 
the place of the validity condition. However, as we will see in
Section~\ref{subsec:validity}, in the context of MAP-learners and 
their teachers, both conditions lead essentially to the same 
results. In Section~\ref{subsec:monotonicity}, we discuss two 
natural monotonicity properties and thereafter,
in Section~\ref{subsec:modes}, we note the equivalence 
of $(O,R)$- and the $(\ol{O},R)$-mode and prove the pairwise 
incomparability of the 
modes $(\ol{O},R)$, $(O,\ol{R})$ and $(\ol{O},\ol{R})$.

\subsection{Validity and Strong Validity} \label{subsec:validity}

We will refer to 
\[ c \not\models z \dund P(z|c) = 0 \]
as the \emph{strong validity condition} 
for the parameters $(P(z|c))_{z \in Z,c \in C}$.
This is the condition that the authors of~\cite{IJCAI}
had imposed on the $c$-conditional likelihoods associated 
with a MAP-learner. We will see that each $L$ satisfying 
the validity condition has a ``close relative'' $L_\ve$ 
that satisfies the strong validity condition. Here comes
the definition of $L_\ve$: 

\begin{definition}[$\ve$-Shift] \label{def:eps-shift}
Let $L$ be given by parameters $P(c)$ and $P(z|c)$ with $c \in C$ 
and $z \in Z$ such that the validity condition is satisfied
but the strong validity condition is not. We say that $L_\ve$ 
(with $0 < \ve \le 1/2$) is the \emph{$\ve$-shift of $L$} 
if $L_\ve$ is given by the parameters $P(c)$ and $P_\ve(z|c)$ 
where
\[
P_\ve(z|c) = \left\{ \begin{array}{ll}
    (1-\ve) \cdot P(z|c) & \mbox{if $z \in Z^+_c(L)$} \\
    \frac{\ve}{|Z_c \sm Z_c^+(L)|} & \mbox{if $z \in Z_c \sm Z_c^+(L)$} \\
    0 & \mbox{if $z \in Z \sm Z_c$}
    \end{array} \right. \enspace .
\]
For convenience, we set $P_\ve(z|c) = P(z|c)$ if already $L$ 
satisfies the strong validity condition.
\end{definition}

Note that $L_\ve$ satisfies the strong validity condition 
because $P_\ve(z|c) = 0$ iff $z \not\in Z_c$ 
and $Z_c = \{z \in Z: c \models z\}$. A learner and its $\ve$-shift 
are related as follows:

\begin{lemma} \label{lem:validity}
Let $L$ be a MAP-learner for $C$ whose parameters satisfy the
validity condition. Then the following holds for 
each $(\alpha,\beta) \in \{O,\ol{O}\} \times \{R,\ol{R}\}$ and 
all sufficiently small $\ve>0$: each teacher for $L$ in sampling
mode $(\alpha,\beta)$ is also a teacher for $L_\ve$ in sampling
mode $(\alpha,\beta)$.
\end{lemma}

\begin{proof}
Suppose that $L$ and $L_\ve$ are both in sampling mode $(\alpha,\beta)$.
Consider a teacher $T$ for $L$. We claim that the following holds:
\begin{equation} \label{eq:convergence}
\forall c_0,c \in C:
\lim_{\ve \ra 0}P_\ve^{\alpha,\beta}(T(c_0)|c) = P^{\alpha,\beta}(T(c_0)|c)
\enspace .
\end{equation}
This would imply that, for every $c_0 \in C$ and sufficiently small $\ve$,
we have
\[ 
c_0 = \arg!max_{c \in C}P^{\alpha,\beta}(T(c_0)|c)
= \arg!max_{c \in C}P_\ve^{\alpha,\beta}(T(c_0)|c) \enspace ,
\]
which, in turn, implies that $T$ is a teacher for $L_\ve$.
We still have to verify~(\ref{eq:convergence}). This can be done
by means of a simple continuity argument. Note first that
\[ 
\forall c \in C, z \in Z: \lim_{\ve \ra 0}P_\ve(z|c) = P(z|c)
\enspace .
\]
Since $P_\ve^{\alpha,R}(T(c_0)|c)$ is a polynomial (and hence a continuous 
function) in the variables $P_\ve(z|c)$ with $z \in T(c_0)$, we may
conclude that~(\ref{eq:convergence}) is true in case of $\beta = R$.
Suppose now that $(\alpha,\beta) = (O,\ol{R})$ 
and $T(c_0) = (z_1,\ldots,z_n)$, which implies that $n \le N^+(C,L)$
and $z_1,\ldots,z_n \in Z_c^+(L)$. 
The function 
\[
P_\ve^{O,\ol{R}}(T(c_0)|c) =
\prod_{i=1}^{n}\frac{P_\ve(z_i|c)}{1-(P_\ve(z_1|c) +\ldots+ P_\ve(z_{i-1}|c))}
\]
is a rational function in the variables $P_\ve(z_i|c)$ for $i=1,\ldots,n$. 
Hence we can apply the continuity argument again but, in addition, we must 
rule out that the denominator, $1-(P_\ve(z_1|c) +\ldots+ P_\ve(z_{i-1}|c))$,
converges to $0$ when $\ve$ approaches $0$. This, however, can be ruled out 
as follows:
\begin{itemize}
\item
Set $\rho := \frac{1}{2} \cdot \min_{c \in C , z \in Z_c^+(L)}P(z|c)$ 
and note that 
$0 < \rho \le \min_{c \in C , z \in Z_c^+(L)}P_\ve(z|c)$. The latter
inequality holds because of $P_\ve(z|c) = (1-\ve) \cdot P(z|c)$ 
and $\ve \le 1/2$. 
\item
Because of $n \le N^+(C,L)$, the set $\{z_1,\ldots,z_{n-1}\}$ cannot 
contain all elements of $Z_c^+(L)$.
\item
Therefore $1-(P_\ve(z_1|c) +\ldots+ P_\ve(z_{i-1}|c) \ge \rho$
for all $i = 1,\ldots,n$ and the limit for $\ve \ra 0$ cannot 
be equal to $0$.
\end{itemize}
We may therefore conclude 
that~(\ref{eq:convergence}) is true in case 
of $(\alpha,\beta) = (O,\ol{R})$. The proof in case 
of $(\alpha,\beta) = (\ol{O},\ol{R})$ is similar.
\end{proof}

\begin{corollary} \label{cor:validity}
With the notation from Definition~\ref{def:eps-shift}, we have
\[ \MAPTD^{\alpha,\beta}_L(C) = \MAPTD^{\alpha,\beta}_{L_\ve}(C) \]
for all sufficiently small $\ve > 0$. 
\end{corollary}

\subsection{Monotonicity Properties} \label{subsec:monotonicity}

It is clear, intuitively, that adding concepts without adding observations
should make the teaching problem harder. Conversely, adding observations
without adding concepts should make the teaching problem easier. In this
section, we formalize these statements and prove them. All results in this
section are formulated in terms of $\MAPTD$. But the corresponding results
with $\MLETD$ at the place of $\MAPTD$ hold es well.

We say that $(C',Z',\models')$ is an \emph{extension} of $(C,Z,\models)$
if $C \seq C'$, $Z \seq Z'$ and, for all $c \in C$ and $z \in Z$,
we have that $c \models z$ if and only if $c \models' z$.

So far, we used a notation 
(e.g.~$\MAPTD^{\alpha,\beta}(C)$ instead of $\MAPTD^{\alpha,\beta}(C,Z,\models)$)
which made a dependence on $(C,Z,\models)$ explicit for $C$ only 
(because the corresponding $Z$ and the corresponding relation $\models$ 
were clear from context). In this section, there is some danger of confusion 
and, consequently, we use a notation which makes the dependence on the whole 
triple $(C,Z,\models)$ more explicit.

\begin{definition} \label{def:L-restriction}
Let $(C',Z',\models')$ be an extension of $(C,Z,\models)$ with $Z'=Z$.
Let $L$ be a MAP-learner for $(C',Z,\models')$ with parameters $P(c')>0$
and $P(z|c')$ for $c' \in C'$ and $z \in Z$. Set $P(C) = \sum_{c \in C}P(c)$.
The MAP-learner with parameters $P(c)/P(C)$ and $P(z|c)$ for $c \in C$ 
and $z \in Z$, denoted by $L_{\da C}$, is called the \emph{restriction of $L$ 
to subclass $C$}. 
\end{definition}

The parameters of a MAP-learner $L$ for $(C',Z,\models')$ must satisfy 
the validity condition. Clearly the parameters of $L_{\da C}$ 
satisfy the validity condition too. Moreover, for each $c \in C$, we have 
that $Z_c^+(L_{\da C}) = Z_c^+(L)$. These observations can be used for
showing the following result:

\begin{lemma}[Concept-Class Monotonicity] \label{lem:class-monotonicity}
With the assumptions and notation as in Definition~\ref{def:L-restriction}, 
the following holds for each sampling mode $(\alpha,\beta)$:
\[
\MAPTD_{L_{\da C}}^{\alpha,\beta}(C,Z,\models)\ \le\   
\MAPTD_{L}^{\alpha,\beta}(C',Z,\models')
\enspace .
\]
\end{lemma}

\begin{proof}
Let $T:C' \ra \cZ^{\alpha,\beta}$ be a teacher for $L$
and let $T_{\da C}$ denote its restriction to subclass $C$.
Clearly the order of $T_{\da C}$ is upper-bounded by the order of $T$.
It suffices to show that $T_{\da C}$ is a teacher for $L_{\da C}$. 
To this end, we have to show the following:
\begin{itemize}
\item[(a)]
If $\beta = \ol{R}$ then, for all $c \in C$, we have
that $|T_{\da C}(c)| \le N^+(C,L_{\da C})$.
\item[(b)]
For all $c_0 \in C$, $c \in C \sm \{c_0\}$, we have that 
$P(c) \cdot P^{\alpha,\beta}(T_{\da C}(c_0)|c) < 
P(c_0) \cdot P^{\alpha,\beta}(T_{\da C}(c_0)|c_0)$.
\end{itemize}
Of course, since $T$ is teacher for $L$, we know that the following hold:
\begin{itemize}
\item[(a')]
If $\beta = \ol{R}$ then, for all $c' \in C'$, we have
that $|T(c')| \le N^+(C',L)$.
\item[(b')]
For all $c'_0 \in C'$, $c' \in C' \sm \{c'_0\}$, we have that 
$P(c') \cdot P^{\alpha,\beta}(T(c'_0)|c') < 
P(c'_0) \cdot P^{\alpha,\beta}(T(c'_0)|c'_0)$.
\end{itemize}
The following calculation verifies (a) under the assumption 
that $\beta = \ol{R}$:
\begin{eqnarray*}
|T_{\da C}(c)| & = & |T(c)|\ \le\ N^+(C',L)\ =\  \min_{c' \in C'}|Z_{c'}^+(L)| \\
& \le & \min_{c \in C}|Z_c^+(L)|\ =\ \min_{c \in C}|Z_c^+(L_{\da C})|\ =\  
N^+(C,L_{\da C}) \enspace .
\end{eqnarray*}
Suppose that $c_0 \in C$ and $c \in C \sm \{c_0\}$. Then (b) can 
be verified as follows:
\[
P(c) \cdot P^{\alpha,\beta}(T_{\da C}(c_0)|c) = 
P(c) \cdot P^{\alpha,\beta}(T(c_0)|c) \le 
P(c_0) \cdot P^{\alpha,\beta}(T(c_0)|c_0) 
= P(c_0) \cdot P^{\alpha,\beta}(T_{\da C}(c_0)|c_0)
\enspace .
\] 
Here the first and the last equation hold because $c_0 \in C$
and therefore $T_{\da C}(c_0) = T(c_0)$.
\end{proof}

\begin{corollary} \label{cor:class-monotonicity}
If $(C',Z',\models')$ is an extension of $(C,Z,\models)$ with $Z = Z'$,
then
\[
\MAPTD^{\alpha,\beta}(C,Z,\models)\ \le\ \MAPTD^{\alpha,\beta}(C',Z,\models')
\enspace .
\] 
\end{corollary}

\begin{definition} \label{def:L-extension}
Let $(C',Z',\models')$ be an extension of $(C,Z,\models)$ with $C'=C$.
Let $L$ be a MAP-learner for $(C,Z,\models)$ with parameters $P(c)$
and $P(z|c)$ for $c \in C$ and $z \in Z$. The MAP-learner with
parameters $P_{\ua Z'}(c) = P(c)$ and
\[
P_{\ua Z'}(z'|c) = \left\{ \begin{array}{ll}
    P(z'|c) & \mbox{if $z' \in Z$} \\
    0       & \mbox{otherwise}
           \end{array} \right. \enspace ,  
\]
denoted by $L_{\ua Z'}$, is called the \emph{extension of $L$ 
to superset $Z'$}.
\end{definition}

The parameters of a MAP-learner $L$ for $(C,Z,\models)$ must satisfy 
the validity condition. It is easy to check that, therefore, the 
parameters of $L_{\ua Z'}$ satisfy the validity condition too. 
Moreover, for each $c \in C$, we have that 
\[
\{z' \in Z': P_{\ua Z'}(z'|c) > 0\}\ =\ \{z \in Z: P(z|c) > 0\}\ 
=\   Z_c^+(L)\}  \enspace ,
\]
which implies that $N^+(C,L_{\ua Z'}) = N^+(C,L)$. These observations 
can be used for showing the following result:

\begin{lemma}[Observation-Set Monotonicity] \label{lem:observation-monotonicity}
With the assumptions and the notation as in Definition~\ref{def:L-extension}, 
the following holds for each sampling mode $(\alpha,\beta)$:
\[
\MAPTD_L^{\alpha,\beta}(C,Z,\models)\ \ge\   
\MAPTD_{L \ua Z'}^{\alpha,\beta}(C,Z',\models') \enspace .  
\]
\end{lemma}

\begin{proof}
Let $T:C \ra \cZ^{\alpha,\beta}$ be a teacher for $L$.
It is sufficient to show that $T$ is also a teacher for $L_{\ua Z'}$
(albeit a teacher for $L_{\ua Z'}$ who does not make use of observations
in $Z' \sm Z$). To this end, we have to show the following:
\begin{itemize}
\item[(a)]
If $\beta = \ol{R}$ then, for all $c \in C$, we have
that $|T(c)| \le N^+(C,L_{\ua Z'})$.
\item[(b)] 
For all $c_0 \in C$, $c \in C \sm \{c_0\}$, we have that 
$P(c) \cdot P_{\ua Z'}^{\alpha,\beta}(T(c_0)|c) < 
P(c_0) \cdot P_{\ua Z'}^{\alpha,\beta}(T(c_0)|c_0)$.
\end{itemize}
Assertion (a), assuming $\beta = \ol{R}$, is obtained by
\[ 
|T(c)| \le N^+(C,L)\ =\ N^+(C,L_{\ua Z'}) \enspace ,
\]
where the first inequality holds because $T$ is a teacher for $L$.
Suppose that $c_0 \in C$ and $c \in C \sm \{c_0\}$. Assertion (b) 
is obtained by
\[
P(c) \cdot P_{\ua Z'}^{\alpha,\beta}(T(c_0)|c) = 
P(c) \cdot P^{\alpha,\beta}(T(c_0|c) <
P(c_0) \cdot P^{\alpha,\beta}(T(c_0)|c_0) = 
P(c_0) \cdot P_{\ua Z'}^{\alpha,\beta}(T(c_0)|c_0)
\enspace ,
\]
where the first and the last equation holds because $T(c_0) \seq Z$
so that the likelihoods of observations in $Z' \sm Z$ do not come
into play. The inequality in the middle holds because $T$ is a teacher
for $L$.
\end{proof}

\begin{corollary} \label{cor:observation-monotonicity}
If $(C',Z',\models')$ is an extension of $(C,Z,\models)$ with $C=C'$,
then
\[
\MAPTD^{\alpha,\beta}(C,Z,\models)\ \ge\ \MAPTD^{\alpha,\beta}(C,Z',\models')
\enspace .
\] 
\end{corollary}

\subsection{A Comparison of the Sampling Modes} \label{subsec:modes}

We say that the sampling mode $(\alpha,\beta)$ \emph{dominates}
the sampling mode $(\alpha',\beta')$ if, for every
concept class $C$ and every MAP-learner $L$ for $C$,
we have that $\MAPTD_L^{\alpha,\beta}(C) \le \MAPTD_L^{\alpha',\beta'}(C)$.
We say they are \emph{equivalent} if they mutually dominate each other, 
i.e., if $\MAPTD_L^{\alpha,\beta}(C) = \MAPTD_L^{\alpha',\beta'}(C)$
holds for every choice of $C$ and $L$.  We say, they are \emph{incomparable} 
if none of them dominates the other one. We start with an easy observation:

\begin{remark} \label{rem:equivalent-modes}
The sampling modes $(O,R)$ and $(\ol{O},R)$ are equivalent. 
\end{remark} 

\begin{proof}
Consider a concept class $C$ and a MAP-learner $L$ for $C$.
Let $\vec{a} \in Z^n$ be a sequence of $k$ distinct elements
with multiplicities $n_1,\ldots,n_k$, respectively. Denote by $A$ 
the corresponding multiset. An inspection of~(\ref{eq:R-probs}) 
shows that the following holds for each $c \in C$:
\begin{equation} \label{eq:multiset-sequence} 
P^{\ol{O},R}(A|c) = 
\frac{n!}{n_1! \ldots n_k!} \cdot P^{O,R}(\vec{a}|c) \enspace .
\end{equation}
Let $\vec{a'}$ be a sequence obtained from $\vec{a}$ by a permutation
of the components. Since $\vec{a'}$ also consists of $k$ distinct
elements with multiplicities $n_1,\ldots,n_k$, respectively,
equation~(\ref{eq:multiset-sequence}) also holds with $\vec{a'}$
at the place of $\vec{a}$. It therefore easily follows that 
a teacher $T$ for $L$, with $L$ being in  sampling mode $(O,R)$, 
can be converted into a teacher $T'$ of the same order for $L$
with $L$ being in sampling mode $(\ol{O},R)$, and vice versa:
\begin{itemize}
\item 
Suppose that $T$ is given. If $T(c) = \vec{a}$, then define $T'(c) = A$
where $A$ is the multiset induced by $\vec{a}$. 
\item
Suppose that $T'$ is given. If $T'(c) = A$ then define $T(A) = \vec{a}$
where $\vec{a}$ is an (arbitrarily chosen) sequence containing the
same elements as $A$ with the same multiplicities.
\end{itemize}
It follows from this discussion 
that $\MAPTD^{O,R}_L(C) = \MAPTD^{\ol{O},R}_L(C)$, which concludes
the proof.
\end{proof}

\begin{corollary} \label{cor:equivalent-modes}
$\MAPTD^{O,R}(C) = \MAPTD^{\ol{O},R}(C)$ and
$\MLETD^{O,R}(C) = \MLETD^{\ol{O},R}(C)$.
\end{corollary}

\noindent
We now turn our attention to the incomparability results:

\begin{theorem} \label{th:incomparable-modes}
The sampling modes $(O,R)$, $(O,\ol{R})$ and $(\ol{O},\ol{R})$
are pairwise incomparable.
\end{theorem}

In order to prove the theorem, we will consider triples $(C,Z,\models)$
with $C = \{c_1,c_2,c_3\}$, $Z = \{z_1,z_2,z_3\}$
and $c_i \models z_j$ for all $1 \le i,j \le 3$.
An important role will be played by concepts of the
form $c^{\pm\Delta}$ with parameters given by
\begin{equation} \label{eq:plus-minus-Delta}
P(z_1|c^{\pm\Delta}) = p+\Delta\ ,\ P(z_2|c^{\pm\Delta}) = p-\Delta\
\mbox{ and }\ P(z_3|c^{\pm\Delta}) = 1-2p \enspace .
\end{equation}
The following Facts 1--4, which pave the way for the proof of
Theorem~\ref{th:incomparable-modes}, can be proven by using the 
derivation rules of analysis. For sake of completeness, these proofs 
are given in the appendix.

\begin{description}
\item[Fact 1:]
Suppose that $0 \le |\Delta| < p < 1/2$. 
Let $c^{\pm\Delta}$ be the concept given by~(\ref{eq:plus-minus-Delta}).
Then $P^{O,R}(z_1,z_2)|c^{\pm\Delta})$
and $P^{\ol{O},\ol{R}}(z_1,z_2|c^{\pm\Delta})$ are both
strictly decreasing when $|\Delta|$ is increased, which
implies that $\Delta = 0$ is the unique maximizer. 
\item[Fact 2:]
Suppose that $0 \le |\Delta| < p < 1/2$. Let $c^{\pm\Delta}$ be 
the concept given by~(\ref{eq:plus-minus-Delta}). Then
\begin{equation} \label{eq1:odd-split}
P^{O,\ol{R}}(z_1,z_2 | c^{\pm\Delta}) - P^{O,\ol{R}}(z_1,z_2 | c^{\pm0})
\left\{ \begin{array}{ll}
   = 0 & \mbox{if $\Delta \in \{0,\frac{p^2}{1-p}\}$} \\
   > 0 & \mbox{if $0 < \Delta < \frac{p^2}{1-p}$} \\
   < 0 & \mbox{otherwise}
\end{array} \right. \enspace .
\end{equation}
\item[Fact 3:]
Suppose that $0 \le \Delta < p < 1/2$. Let $c^{\pm\Delta}$ be 
the concept given by~(\ref{eq:plus-minus-Delta}). Then
\begin{equation} \label{eq2:odd-split}
P^{O,R}(z_1,z_1,z_2 | c^{\pm\Delta}) - P^{O,R}(z_1,z_1,z_2 | c^{\pm0})
\left\{ \begin{array}{ll}
   = 0 & \mbox{if $\Delta \in \left\{0,\frac{1}{2}(\sqrt{5}-1)p\right\}$} \\
   > 0 & \mbox{if $0 < \Delta < \frac{1}{2}(\sqrt{5}-1)p$} \\
   < 0 & \mbox{otherwise}
\end{array} \right. \enspace .
\end{equation}
\item[Fact 4:]
Suppose that $0 < p < 1/2$ and $1 \le t < \frac{1-p}{p}$.
Let $c^{(t)}$ be the concept given by
\begin{equation} \label{eq:times-t}
c^{(t)}(z_1) = pt\ ,\ c^{(t)}(z_2) = p/t\ \mbox{ and }\ 
c^{(t)}(z_3) = 1 - pt - p/t \enspace .
\end{equation}
Then $P^{\ol{O},\ol{R}}(z_1,z_2|c^{(t)})$ is strictly increasing
with $t$.
\end{description}

A couple of more intuitive remarks are in place here.
Fact~1 tells us that, in sampling modes $(O,R)$ and $(\ol{O},\ol{R})$, 
a concept explains observations $z_1,z_2$ the better (in the maximum
likelihood sense), the more evenly it splits the available probability 
mass $2p$ among them. We will refer to an application of Fact~1 as 
applying the ``even-split argument''. 
In sampling mode $(O,\ol{R})$, however, the even split does not maximize
the likelihood of these observations. The likelihood of $z_1,z_2$ 
becomes larger if the probability assigned to $z_1$ is slightly larger
than the probability assigned to $z_2$. See (\ref{eq1:odd-split}). 
A similar remark applies to the sampling mode $(O,R)$ and the 
sequence $z_1,z_1,z_2$.  See~(\ref{eq2:odd-split}). Fact~4 is concerned 
with sampling mode $(\ol{O},\ol{R})$ and a multiplicative decomposition 
of $p^2$ into $pt$ (the probability assigned to $z_1$) and $p/t$ (the 
probability assigned to $z_2$) with $t \ge 1$. According to Fact~4, 
the likelihood of $\{z_1,z_2\}$ becomes larger when the scaling 
factor $t\ge1$ is increased. Note that this is not in contradiction
with the even-split argument, because $pt + p/t$ is itself strictly 
increasing with $t$ so that the even-split argument does not apply.

We would furthermore like to note that the $c$-conditional likelihood 
of a (multi-)set or sequence of observations becomes larger if one of 
the relevant $c$-conditional likelihood parameters is increased while 
the others are fixed. We refer to this way of arguing as applying the
``monotonicity argument''. 

\noindent
Theorem~\ref{th:incomparable-modes} is a 
direct consequence of the following three lemmas.

\begin{lemma} \label{lem1:incomporability}
Consider the triple $(C,Z,\models)$ 
with $C = \{c_1,c_2,c_3\}$, $Z = \{z_1,z_2,z_3\}$ 
and $c_i \models z_j$ for all $1 \le i,j \le 3$.
Let $L$ be an MLE-learner for $C$ with parameters 
given by 
\[
\begin{array}{|c||c|c|c|c|}
\hline
P(z|c) & c_1 & c_2 & c_3 \\
\hline
\hline
z_1 & p+\Delta_1 & p+\Delta_2 & p \\
\hline
z_2 & p-\Delta_1 & p-\Delta_2 & p \\
\hline
z_3 & 1-2p & 1-2p & 1-2p \\
\hline
\end{array}
\enspace ,
\]
where 
$0 < \Delta_1 < \frac{p^2}{1-p} < \Delta_2 = 
\frac{1}{2}(\sqrt{5}-1)p < p \le 0.4$.\footnote{The constraint $p \le 0.4$ 
has the effect that $\frac{p}{1-p} < \frac{1}{2}(\sqrt{5}-1)$.}
Then
\begin{equation} \label{eq1:incomparability}
\MLETD^{O,R}_L(C) = 3\ ,\ \MLETD^{O,\ol{R}}_L(C) = 2\ \mbox{ and }
\MLETD^{\ol{O},\ol{R}}_L(C) = \infty \enspace .
\end{equation}
\end{lemma}

\begin{proof}
It is obvious that, in any mode of sampling, the concept $c_2$ 
can be taught by observation $z_1$ and the concept $c_3$ can be 
taught by observation $z_2$.  An inspection of~(\ref{eq1:odd-split})
and~(\ref{eq2:odd-split}) reveals that
\begin{eqnarray*} 
P_L^{O,\ol{R}}(z_1,z_2|c_1) & > & P_L^{O,\ol{R}}(z_1,z_2|c_3) >
P_L^{O,\ol{R}}(z_1,z_2|c_2) \enspace , \\
P_L^{O,R}(z_1,z_1,z_2|c_1) & > & P_L^{O,R}(z_1,z_1,z_2|c_2) =  
P_L^{O,R}(z_1,z_1,z_2|c_3) \enspace .
\end{eqnarray*} 
It follows that $c_1$ can be taught in $(O,\ol{R})$-mode 
(resp.~in $(O,R)$-mode) by the sequence $z_1,z_2$ (resp.~by the 
sequence $z_1,z_1,z_2$). We will argue now that there are no shorter
sequences for teaching $c_1$ and that, in $(\ol{O},\ol{R})$-mode,
$c_1$ cannot be taught at all. An application of the monotonicity
argument yields that $c_1$ cannot be taught by a single observation 
(regardless of the sampling mode). The same remark holds for $2$ 
observations except, possibly, for observations $z_1,z_2$. But, 
by the even-split argument, it is the concept $c_3$ that assigns 
the highest probability to the sequence $(z_1,z_2) \in \cZ^{O,R}$ 
resp.~to the set $\{z_1,z_2\} \in \cZ^{\ol{O},\ol{R}}$. 
Thus $(O,\ol{R})$ is the only sampling mode in which $c_1$ can be 
taught by $2$ observations. It follows that, in $(\ol{O},\ol{R})$-mode, 
$c_1$ cannot be taught at all.\footnote{Here we make use of 
the fact that, if $Z_c = Z$ for each $c \in C$, 
then $P^{\ol{O},\ol{R}}(Z|c) = 1$ for each $c \in C$. 
Note that this rules out the possibility of having teaching sets 
of size $3 = |Z|$.} We may conclude from this discussion 
that the identities in~(\ref{eq1:incomparability}) are valid,
\end{proof}

Lemma~\ref{lem1:incomporability} implies that $(O,R)$ does not 
dominate $(O,\ol{R})$ and $(\ol{O},\ol{R})$ does not dominate 
any of the other sampling modes. The next result leads to some 
more no-domination results:

\begin{lemma} \label{lem2:incomparability}
Consider the triple $(C,Z,\models)$ 
with $C = \{c_1,c_2,c_3\}$, $Z = \{z_1,z_2,z_3\}$ 
and $c_i \models z_j$ for all $1 \le i,j \le 3$.
Let $L$ be an MLE-learner for $C$ with the 
parameters $P(z|c)$ given by
\[
\begin{array}{|c||c|c|c|c|}
\hline
P(z|c) & c_1 & c_2 & c_3 \\
\hline
\hline
z_1 & p & p+\Delta & p-\Delta \\
\hline
z_2 & p & p-\Delta & p+\Delta \\
\hline
z_3 & 1-2p & 1-2p & 1-2p \\
\hline
\end{array} 
\enspace , 
\]
where $0 < \Delta < \frac{p^2}{1-p} < p < 1/2$. Then
\begin{equation} \label{eq2:incomparability}
\MLETD^{O,R}_L(C) = \MLETD^{\ol{O},\ol{R}}_L(C) = 2\ \mbox{ and }\ 
\MLETD^{O,\ol{R}}_L(C) = \infty \enspace . 
\end{equation}
\end{lemma}

\begin{proof}
Clearly the concept $c_2$ can be taught by observation $z_1$ 
and the concept $c_3$ can be taught by observation $z_2$ 
in any mode of sampling. The concept $c_1$ cannot be taught by
a single observation. But it can be taught by the 
sequence $(z_1,z_2)$ in $(O,R)$-mode and by the set $\{z_1,z_2\}$ 
in $(\ol{O},\ol{R})$-mode (application of the even-split argument). 
We finally discuss teachability of $c_1$ in $(O,\ol{R})$-mode. An 
application of the monotonicity argument yields that $c_1$ cannot
be taught in $(O,\ol{R})$-mode by two observations except, possibly, 
by the observations $(z_1,z_2)$ or $(z_2,z_1)$ in $\cZ^{O,\ol{R}}$. 
But an inspection of~(\ref{eq1:odd-split}) reveals that it is the 
concept $c_2$ (resp.~$c_3$) that assigns the highest probability 
to $(z_1,z_2)$ (resp.~to $(z_2,z_1)$). It follows
that, in $(O,\ol{R})$-mode, the concept $c_1$ cannot be taught
at all. We may conclude from this discussion that the identities
in~(\ref{eq2:incomparability}) are valid.
\end{proof}

Lemma~\ref{lem2:incomparability} implies that $(O,\ol{R})$ does not 
dominate any of the other sampling modes. The next result implies
$(O,R)$ does not dominate $(\ol{O},\ol{R})$.

\begin{lemma} \label{lem3:incomarability}
Consider the triple $(C,Z,\models)$  
with $C = \{c_1,c_2,c_3\}$, $Z = \{z_1,z_2,z_3\}$ 
and $c_i \models z_j$ for all $1 \le i,j \le 3$.
Let $L$ be an MLE-learner for $C$ with parameters $P(z|c)$
given by
\[
\begin{array}{|c||c|c|c|c|}
\hline
P(z|c) & c_1 & c_2 & c_3 \\
\hline
\hline
z_1 & sp & p & sp+\ve \\
\hline
z_2 & p/s & p & p/s-\ve \\
\hline
z_3 & 1-sp-p/s & 1-2p & 1-sp-p/s \\
\hline
\end{array} 
\enspace ,
\]
where $0 < p <\frac{1}{2}$ and $1 < s \le \frac{1-p}{p}$.
Then
\begin{equation} \label{eq3:incomparability}
\MLETD^{\ol{O},\ol{R}}_L(C) = 2 < \MLETD^{O,R}_L(C) \enspace ,
\end{equation}
provided that $\ve>0$ is sufficiently small. 
\end{lemma}

\begin{proof}
Clearly, the concept $c_2$ can be taught by observation $z_2$ 
and $c_3$ can be taught by observation $z_1$ in any mode of sampling. 
It is obvious that $c_1$ cannot be taught by a single observation
(regardless of the sampling mode). In $(O,R)$-mode, the concept $c_1$ 
cannot be taught by sequences of length $2$ because $c_1$ is for none
of them the unique maximizer:
\begin{itemize}
\item 
$P^{O,R}_L(z_1,z_2|c_1) = p^2 =P^{O,R}_L(z_1,z_2|c_2)$.
\item
$P^{O,R}_L(z_1,z_3|c_1) < P^{O,R}_L(z_1,z_3|c_3)$ and
$P^{O,R}_L(z_2,z_3|c_1) < P^{O,R}_L(z_2,z_3|c_2)$.\footnote{These
are two applications of the monotonicity argument. 
Note that $s + \frac{1}{s} > 2$ for all $s>1$.}
\end{itemize}
However, in $(\ol{O},\ol{R})$-mode, the concept $c_1$ can be taught
by the set $\{z_1,z_2\}$:
\begin{itemize}
\item
Concept $c_1$ distributes the probability mass $sp + p/s$ (slightly)
more evenly on $z_1$ and $z_2$ than the concept $c_3$. By the
even-split argument, we obtain
$P^{\ol{O},\ol{R}}(\{z_1,z_2\}|c_1) > P^{\ol{O},\ol{R}}(\{z_1,z_2\}|c_3)$.
\item
Recall from Fact~4 that $c^{(t)}$, with $t \ge 1$, denotes the concept 
which assigns probability $pt$ to $z_1$, probability $p/t$ to $z_2$ 
and the remaining probability mass to $z_3$. Note that $c_1 = c^{(s)}$
and $c_2 = c^{(1)}$. According to Fact~4, the 
function $P^{\ol{O},\ol{R}}(z_1,z_2 | c^{(t)})$ is strictly increasing
with $t$. Hence 
$P^{\ol{O},\ol{R}}(\{z_1,z_2\}|c_1) > P^{\ol{O},\ol{R}}(\{z_1,z_2\}|c_2)$.
\end{itemize}
The identities in~(\ref{eq3:incomparability}) are immediate
from this discussion. 
\end{proof}

\noindent
Putting the above three lemmas together, we obtain 
Theorem~\ref{th:incomparable-modes}.

\section{MAP-Based Teaching and Saturating Matchings}
\label{sec:smn}

Suppose that $C$ is a concept class with observation set $Z$ 
and consistency relation $\models$. The bipartite graph $G(C) = (C,Z,E)$ 
with 
\[ E = \{(c,z) \in C \times Z: c \models z\} \]
is called the \emph{consistency graph (associated with $C$)}. 
Let $\cZ^{\alpha,\beta}$ 
with $(\alpha,\beta) \in \{O,\ol{O}\} \times \{R,\ol{R}\}$ be
the notation that was introduced in Section~\ref{subsec:sampling}.
The bipartite 
graph $G(C)^{\alpha,\beta} = (C,\cZ^{\alpha,\beta},E^{\alpha,\beta})$ 
with 
\[
E^{\alpha,\beta} = \{(c,\zeta) \in C \times \cZ^{\alpha,\beta}: c \models \zeta\}
\]
is called the \emph{extended consistency graph (associated with $C$)}.
The graph resulting from $G(C)^{\alpha,\beta}$ by the removal of 
the vertex $\eset$ from the second vertex class $\cZ^{\alpha,\beta}$
will be denoted by $G(C)^{\alpha,\beta}_{\neq\eset}$.
We denote by $\SMN(G(C)^{\alpha,\beta})$ the smallest possible
order of a $C$-saturating matching in $G(C)^{\alpha,\beta}$.
Analogously, $\SMN(G(C)^{\alpha,\beta}_{\neq\eset})$ denotes 
the smallest possible order of a $C$-saturating matching 
in $G(C)^{\alpha,\beta}_{\neq\eset}$.
For ease of later reference, we make the following observation:

\begin{remark} \label{rem:smn-lb}
Suppose that $T:C \ra \cZ^{\alpha,\beta}$ is a mapping which satisfies 
\begin{equation} \label{eq:injectivity}
\forall c,c' \in C: (c \models T(c)) \wedge (c \neq c' \impl T(c) \neq T(c'))
\enspace .
\end{equation}
Then $T$ is of order at least $\SMN(G(C)^{\alpha,\beta})$.
Moreover, if $T$ satisfies~(\ref{eq:injectivity}) and $\eset$
is not in the image of $T$, then $T$ is of order 
at least $\SMN(G(C)^{\alpha,\beta}_{\neq\eset})$.
\end{remark}

\begin{proof}
If $T$ satisfies~(\ref{eq:injectivity}), then $T$ represents 
a $C$-saturating matching in $G(C)^{\alpha,\beta}$.
If additionally $\eset$ is not in the image of $T$, then $T$ represents
a $C$-saturating matching in $G(C)^{\alpha,\beta}_{\neq\eset}$.
\end{proof}

\noindent
Here is the main result of this section:

\begin{theorem} \label{th:smn-bounds}
For each sampling mode $(\alpha,\beta)$, we have
\begin{equation} \label{eq:smn-lb} 
\MAPTD^{\alpha,\beta}(C) \ge \SMN(G(C)^{\alpha,\beta})\ \mbox{ and }
\MLETD^{\alpha,\beta}(C) \ge \SMN(G(C)^{\alpha,\beta}_{\neq\eset}) 
\enspace .
\end{equation}
Moreover, for $(\alpha,\beta) = (\ol{O},R)$, this holds with equality.
\end{theorem}

\begin{proof}
Let $L$ be a MAP-learner for $C$ and let $(\alpha,\beta)$ denote
its sampling mode. Let $T$ be a teacher for $L$. Recall 
from~(\ref{eq:general-conditions}) that $T$ satisfies~(\ref{eq:injectivity}).
Moreover, if $L$ is an MLE-learner for $C$, then $T(c) \neq \eset$
for all $c \in C$. Now an application of Remark~\ref{rem:smn-lb}
yields~(\ref{eq:smn-lb}). \\
We move on and prove 
that $\MLETD^{\ol{O},R}(C) \le \SMN(G(C)_{\neq\eset}^{\ol{O},R})$.
Suppose that $M$ is a $C$-saturating matching in $G(C)^{\ol{O},R}_{\neq\eset}$ 
that is of order $\SMN(G(C)^{\ol{O},R}_{\neq\eset})$. For each $c \in C$ 
and $z \in Z$, let $n(z,c)$ denote the number of occurrences of $z$ 
in the multiset $M(c)$ and let $n(c) = |M(c)|$. Consider a learner $L$ 
with uniform priors (= MLE-learner) and 
the parameters $P(z|c) = \frac{n(z,c)}{n(c)}$. Note that these 
parameters satisfy the validity condition.  It suffices to show 
that $M$ represents a teacher for $L$, i.e., we have to show that
\[
\forall c^* \in C: c^* = \arg!max_{c \in C} P^{\ol{O},R}(M(c^*)|c) 
\enspace .
\]
To this end, we pick a concept $c$ from $C \sm \{c^*\}$, and proceed 
by case analysis:
\begin{description}
\item[Case 1:] $M(c^*)$ and $M(c)$ contain the same elements of $Z$ 
(albeit with different multiplicities)\footnote{The multiplicities
cannot be the same because $M: C \ra \cZ^{\ol{O},R}$ is a matching.}. \\
Denote these elements by $z_1,\ldots,z_k$. Let $n := n(c^*)$,
$n_i = n(z_i,c^*)$. Then $p_i := n_i/n$ is the relative frequency 
of $z_i$ in $M(c^*)$. Let $q_i$ denote the relative frequency 
of $z_i$ in $M(c)$, which implies that $\vec{q} \neq \vec{p}$. 
It follows that 
\[
P^{\ol{O},R}(M(c^*)|c^*) = 
\frac{n!}{n_1! \ldots n_k!}\cdot\prod_{i=1}^{k}p_i^{n_i}\ \mbox{ and }\  
P^{\ol{O},R}(M(c^*)|c) = 
\frac{n!}{n_1! \ldots n_k!}\cdot\prod_{i=1}^{k}q_i^{n_i} \enspace .
\]
A straightforward calculation shows 
that $P^{\ol{O},R}(M(c^*)|c^*) > P^{\ol{O},R}(M(c^*)|c)$ iff
\begin{equation} \label{eq:kld}
\sum_{i=1}^{k} p_i \log\left(\frac{p_i}{q_i}\right) > 0 \enspace . 
\end{equation}
The left-hand side is the Kullback-Leibler divergence (= KLD) 
between $\vec{p}$ and $\vec{q}$. Since the KLD is non-negative 
and $0$ only if $\vec{q} = \vec{p}$, the condition~(\ref{eq:kld}) 
is satisfied.
\item[Case 2:] 
$M(c^*)$ contains an element that is not contained in $M(c)$. \\
Then the $c$-conditional likelihood of $M(c^*)$ equals $0$. 
\item[Case 3:] 
All elements in $M(c^*)$ are contained in $M(c)$, but $M(c)$ contains
an element that is not contained in $M(c^*)$. \\
Then the $c$-conditional likelihood of $M(c^*)$ can be expressed
as $\Pr(E_1) \cdot \Pr(E_2 | E_1)$ for the following two events:
\begin{itemize}
\item[$E_1$:] 
$n(c^*)$-fold $c$-sampling yields only elements from $M(c^*)$.
\item[$E_2$:] 
$n(c^*)$-fold $c$-sampling yields $M(c^*)$. 
\end{itemize}
Since $M(c)$ contains an element that is not contained in $M(c^*)$,
we have $\Pr(E_1) < 1$. It follows from the analysis of Case 1 
that $\Pr(E_2 | E_1)$ is upper-bounded by the $c^*$-conditional 
likelihood of $M(c^*)$.
\end{description}
We may conclude from the above discussion 
that $c^* = \arg!max_{c \in C} P^{\ol{O},R}(M(c^*)|c)$. 
Thus $M$ can be seen as a teacher for $L$. It follows 
that $\MLETD^{\ol{O},R}(C) \le \SMN(G(C)^{\ol{O},R}_{\ne\eset})$. \\
The inequality $\MAPTD^{\ol{O},R}(C) \le \SMN(G(C)^{\ol{O},R})$
can be obtained in a similar fashion. We start with a $C$-saturating
matching $M$ in $G(C)^{\ol{O},R}$ that is of order $\SMN(G(C)^{\ol{O},R})$.
If $M$ does not assign $\eset$ to any concept, we can proceed as before.
Otherwise, if $M(c_0) = \eset$ for some $c_0 \in C$, we still use
a similar reasoning but with a slight modification of the parameter 
collection of the learner $L$:
\begin{itemize}
\item 
The priors are given by setting $P(c_0) = \frac{1+\ve}{|C|}$
and by letting the remaining $|C|-1$ concepts evenly share the
remaining probability mass (still almost uniform priors).
\item
The parameters $P(z|c)$ are chosen as before.
\end{itemize}
We can again view the matching $M$ as a teacher for $L$.
Since $P^{\ol{O},R}(\eset|c) = 1$ for all $c \in C$, we obtain
\[
\arg!max_{c \in C}\left(P(c) \cdot  P^{\ol{O},R}(\eset|c)\right) =
\arg!max_{c \in C}P(c) = c_0 \enspace .
\]
For the remaining concepts, the reasoning is as before provided 
that $\ve>0$ s sufficiently small: this is an easy continuity
argument which exploits that the priors converge
to the uniform distribution on $C$ if $\ve$ approaches $0$. 
\end{proof}

\noindent
Clearly 
\begin{eqnarray*} 
\SMN(G(C)^{O,R}) & \le & \min\{\SMN(G(C)^{\ol{O},R}) , \SMN(G(C)^{O,\ol{R}})\} \\
& \le & \max\{\SMN(G(C)^{\ol{O},R}) , \SMN(G(C)^{O,\ol{R}})\} 
\le \SMN(G(C)^{\ol{O},\ol{R}}) 
\end{eqnarray*}
and
\begin{eqnarray*} 
\SMN(G(C)_{\neq\eset}^{O,R}) & \le & 
\min\{\SMN(G(C)_{\neq\eset}^{\ol{O},R}) , \SMN(G(C)_{\neq\eset}^{O,\ol{R}})\} \\
& \le &\max\{\SMN(G(C)_{\neq\eset}^{\ol{O},R}) , \SMN(G(C)_{\neq\eset}^{O,\ol{R}})\} 
\le \SMN(G(C)_{\neq\eset}^{\ol{O},\ol{R}}) \enspace .
\end{eqnarray*}
Combining this with Theorem~\ref{th:smn-bounds} and with 
Corollary~\ref{cor:equivalent-modes}, we immediately obtain
the following result:

\begin{corollary} \label{cor:smn-bounds}
\mbox{}
\begin{enumerate}
\item
$\MAPTD^{\ol{O},R}(C) = \SMN(G(C)^{\ol{O},R}) \le
\SMN(G(C)^{\ol{O},\ol{R}}) \le \MAPTD^{\ol{O},\ol{R}}(C)$. 
\item
$\MLETD^{\ol{O},R}(C) = \SMN(G(C)_{\neq\eset}^{\ol{O},R}) \le
\SMN(G(C)_{\neq\eset}^{\ol{O},\ol{R}}) \le \MLETD^{\ol{O},\ol{R}}(C)$. 
\end{enumerate}
\end{corollary}

Hence we get $\MAPTD^{\ol{O},R}(C) \le \MAPTD^{\ol{O},\ol{R}}(C)$
and $\MLETD^{\ol{O},R}(C) \le \MLETD^{\ol{O},\ol{R}}(C)$ despite
of the fact that $(\ol{O},R)$ does not dominate $(\ol{O},\ol{R})$.

\section{On Concepts Taught by Labeled Examples}
\label{sec:labeled-examples}

In this section, we will restrict ourselves to triples $(C,Z,\models)$ 
of the form as described in Example~\ref{ex:labeled-examples}, i.e.,
$C$ is a family of subsets of a domain $X$, $Z = X \times \{0,1\}$
and $\models$ is given by~(\ref{eq:labeled-examples}). 

We will see that, for each triple $(C,Z,\models)$ of this special 
form and for each sampling mode $(\alpha,\beta)$ except $(O,R)$, 
we have $\MAPTD^{\alpha,\beta}(C) = \SMN(G(C)^{\alpha,\beta})$.
For $(\alpha,\beta) = (\ol{O},R)$, this is already known from
Theorem~\ref{th:smn-bounds}. For the other sampling modes, 
$(\ol{O},\ol{R})$ and $(O,\ol{R})$, it will be shown in 
Section~\ref{subsec:smn},
Since the modes $(O,R)$ and $(\ol{O},R)$
are equivalent, we see that, for triples of the special form,
the MAP-teaching dimensions of $C$ are fully determined by the 
saturating matching numbers associated with $G(C)$.

In Section~\ref{subsec:map-versus-mle} we explore how MAP- and 
MLE-learners are related. For a given collection of conditional 
likelihoods, it can make much of a difference whether we commit 
ourselves to uniform priors or not. However, in the case of optimally
parameterized learners, the freedom for choosing a non-uniform prior
is of minor importance only: it turns out that the MLE-teaching
dimension exceeds the MAP-teaching dimension at most by $1$.

In Section~\ref{subsec:mletd-upper-bounds}, we will see that
the $\MLETD^{\ol{O},\ol{R}}(C)$ is upper bounded by the
so-called antichain number of $C$, by the VC-dimension of $C$ 
and by the no-clash teaching dimension of $C$. 
These upper bounds are then, all the more, valid 
for all parameters $\MAPTD^{\alpha,\beta}(C)$ 
(no matter how he sampling mode $(\alpha,\beta)$) is chosen).

In Section~\ref{subsec:polytime}, we will show 
that the saturating matching numbers associated with $G(C)$
(and hence the MAP-teaching dimensions of $C$) can be computed
in polytime.

\subsection{Saturating Matching Number Revisited} 
\label{subsec:smn}

We start with the two main results of this section.

\begin{theorem} \label{th:maptd-smn-set}
Suppose that $(C,Z,\models)$ is of the form as described 
in Example~\ref{ex:labeled-examples}. 
Then $\MAPTD^{\ol{O},\ol{R}}(C) = \SMN(G(C)^{\ol{O},\ol{R}})$
and $\MLETD^{\ol{O},\ol{R}}(C) = \SMN(G(C)_{\neq\eset}^{\ol{O},\ol{R}})$.
\end{theorem}

\begin{proof}
The $\ge$-direction of the claimed equalities is covered by
Theorem~\ref{th:smn-bounds}. We have to show the $\le$-direction.
We may restrict ourselves to 
proving $\MLETD^{\ol{O},\ol{R}}(C) \le \SMN(G(C)_{\neq\eset}^{\ol{O},\ol{R}})$
because the proof
for $\MAPTD^{\ol{O},\ol{R}}(C) \le \SMN(G(C)^{\ol{O},\ol{R}})$ is quite similar
and uses the same kind of arguments that we had used in the final part
of the proof of Theorem~\ref{th:smn-bounds}. \\
Set $m = |X|$, $d^+ = \SMN(G(C)^{\ol{O},\ol{R}})$
and let $M: C \ra \cZ^{\ol{O},\ol{R}} \sm \{\eset\}$ be a $C$-saturating matching 
in $G(C)^{\ol{O},\ol{R}}$ of order $d^+$.
For every $c \in C$, we set $d(c) = |M(c)|$. Note that $1 \le d(c) \le d^+$.
If $d^+ = m$, then we are done because $\MLETD^{\ol{O},\ol{R}}(C)$ 
cannot exceed $m$. We may assume therefore that $d^+ \le m-1$.
Let $0 < \ve \le \frac{1}{2}$ be a small real number (where the meaning
of ``small'' will become clear from what follows). For each $c \in C$,
we set
\begin{equation} \label{eq:Z-partition}
U_0(c) :=  \{(x,b)\in Z: c(x) \neq b\}\ \mbox{ , }\
U_1(c) := \{(x,b)\in Z: c(x) = b \wedge (x,b) \notin M(c)\}
\end{equation}
and $U(c) = U_0(c) \cup U_1(c)$.
Note that, for each $c \in C$, the set $Z$ partitions
into $M(c)$, $U_0(c)$ and $U_1(c)$. For each $c \in C$ and
each $(x,b) \in Z$, we set
\begin{equation} \label{eq:teaching-by-matching}
P((x,b) | c) = \left\{ \begin{array}{cl}
                \frac{1-\ve}{d(c)} & \mbox{if $(x,b) \in M(c)$} \\
                \frac{\ve}{m-d(c)} & \mbox{if $(x,b) \in U_1(c)$} \\
                0               & \mbox{if $(x,b) \in U_0(c)$}
                \end{array} \right. \enspace .
\end{equation}
Let $L$ be the MLE-learner given by~(\ref{eq:teaching-by-matching}).
We aim at showing that the matching $M:C \ra \cZ^{\ol{O},\ol{R}} \sm \{\eset\}$
can be seen as a teacher for $L$. To this end, it suffices to show 
that the condition
\begin{equation} \label{eq:mle}
\forall c \neq c_0 \in C: 
P^{\ol{O},\ol{R}}(M(c_0)|c_0) > P^{\ol{O},\ol{R}}(M(c_0)|c)
\end{equation}
is satisfied provided that $\ve$ is sufficiently small.
We briefly note that  $|M(c)| + |U_1(c)| = m \ge d^+$ and $\ve \le 1/2$,
and proceed with two claims which will help us to verify~(\ref{eq:mle}).
\begin{description}
\item[Claim 1:]
Call a subset of $Z$ \emph{$c$-rare} if it contains a (low probability)
element from $U(c)$ while missing a (high probability) element
from $M(c)$. Suppose that $d \le d^+$. Then the probability
that $d$-fold $P(\cdot|c)$-sampling without replacement 
leads to a $c$-rare sample is smaller than $d \ve$ divided
by $\frac{1-\ve}{d(c)}$ and, therefore, smaller than $2dd(c)\ve$.
\item[Proof of Claim 1:] The total $P(\cdot|c)$ probability mass of $U(c)$
is $\ve$ whereas any element of $M(c)$ 
has a $P(\cdot,c)$-probability of $\frac{1-\ve}{d(c)}$.
For $k=1,\ldots,d$, let $E_k$ be the event that, within trial $k$,
a point from $U(c)$ is sampled although at least one point from $M(c)$
has not been sampled before. It suffices to upper-bound the
probability of $E_1 \vee\ldots\vee E_d$. The probability of $E_k$ 
is obviously smaller than  $\ve$ divided by $\frac{1-\ve}{d(c)}$
and therefore smaller than $\frac{d(c)\ve}{1-\ve} \le 2d(c)\ve$.
An application of the union bound yields an additional factor~$d$.
\item[Claim 2:]
Suppose that $d \le d(c)$. Then a sample of size $d$ which contains
an element from $U_1(c)$ is $c$-rare (because it necessarily must
miss an element from $M(c)$).
\end{description}
Setting $c = c_0$ and $d = d(c_0)$, we infer from the above claims
that $P^{\ol{O},\ol{R}}(M(c_0)|c_0) > 1-2d(c_0)^2\ve$.
Consider now an arbitrary, but fixed, concept $c_1 \in C\sm\{c_0\}$.
Then $M(c_1) \neq M(c_0)$. We proceed by case analysis:
\begin{description}
\item[Case 1:] Neither $M(c_0) \subset M(c_1)$ nor $M(c_1) \subset M(c_0)$. \\
Then $M(c_0)$ is a $c_1$-rare sample.
Hence $P^{\ol{O},\ol{R}}(M(c_0)|c_1) < 2d(c_0)d(c_1)\ve$.
\item[Case 2:] $M(c_0) \subset M(c_1)$. \\
We apply a symmetry argument. Every sample containing $d(c_0)$
elements of $M(c_1)$ has the same chance for being obtained
from $d(c_0)$-fold $P(\cdot|c_1)$-sampling without replacement.
Hence
\[
P^{\ol{O},\ol{R}}(M(c_0)|c_1) \le \binom{d(c_1)}{d(c_0)}^{-1} \le \frac{1}{d(c_1)}
\le \frac{1}{2} \enspace ,
\]
where the last two inequalities follow from $1 \le d(c_0) \le d(c_1)-1$.
\item[Case 3:] $M(c_1) \subset M(c_0)$. \\
We may assume that $M(c_0) \seq M(c_1) \cup U_1(c_1)$ because, otherwise,
we obtain directly $P^{\ol{O},\ol{R}}(M(c_0)|c_1) = 0$. We apply again a symmetry argument. 
Every sample containing $M(c_1)$ and $d(c_0) - d(c_1)$ elements of $U_1(c_1)$ 
has the same chance for being obtained 
from $d(c_0)$-fold $P(\cdot|c_1)$-sampling without replacement. Hence
\[
P^{\ol{O},\ol{R}}(M(c_0)|c_1) \le \binom{m-d(c_1)}{d(c_0)-d(c_1)}^{-1}
\enspace .
\]
The latter expression is upper-bounded by $\frac{1}{2}$
because $1 \le d(c_0)-d(c_1) < m-d(c_1)$, $d(c_1) \le d(c_0)-1 \le m-2$
and, therefore, $m - d(c_1) \ge 2$.
\end{description}
It becomes obvious from this discussion that condition~(\ref{eq:mle})
is satisfied provided that $\ve$ is sufficiently small.
\end{proof}

\begin{theorem} \label{th:maptd-smn-repfree}
Suppose that $(C,Z,\models)$ is of the form as described 
in Example~\ref{ex:labeled-examples}. 
Then $\MAPTD^{O,\ol{R}}(C) = \SMN(G(C)^{O,\ol{R}})$
and $\MLETD^{O,\ol{R}}(C) = \SMN(G(C)_{\neq\eset}^{O,\ol{R}})$.
\end{theorem}

\begin{proof}
The $\ge$-direction of the claimed equalities is covered by 
Theorem~\ref{th:smn-bounds}. We have to show the $\le$-direction. 
We may restrict ourselves to 
proving $\MLETD^{O,\ol{R}}(C) \le \SMN(G(C)_{\neq\eset}^{O,\ol{R}})$
because the proof
for $\MAPTD^{O,\ol{R}}(C) \le \SMN(G(C)^{O,\ol{R}})$ is quite similar
and uses the same kind of arguments that we had used in the final part 
of the proof of Theorem~\ref{th:smn-bounds}. \\
Set $m = |X|$, $d^+ = \SMN(G(C)_{\neq\eset}^{O,\ol{R}})$ and 
let $M: C \ra \cZ^{O,\ol{R}} \sm \{\eset\}$ be a $C$-saturating matching 
in $G(C)_{\neq\eset}^{O,\ol{R}}$ of order $d^+$.
If $d^+ = m$, then we are done because $\MLETD^{O,\ol{R}}(C)$ cannot 
exceed $m$. We may assume therefore that $d^+ \le m-1$.
For every $c \in C$, we set $d(c) = |M(c)|$. Note that $1 \le d(c) \le d^+$.
We fix for each concept $c \in C$ a sequence $z_1^c,\ldots,z_m^c$ 
consisting of all elements of $Z_c$ subject to the constraint 
that $z_1^c,\ldots,z_{d(c) }^c = M(c)$, i.e., this sequence must start 
with $M(c)$. In the sequel, we will specify the parameter set of
an MLE-learner of $C$. We do this in two stages. In Stage $1$,
we make a preliminary definition which already achieves that each $c^* \in C$
is a (not necessarily unique) maximizer of $P^{O,\ol{R}}(M(c^*|c))$.
In Stage 2, we make some infinitesimal changes of the parameter set
(by bringing a small parameter $\ve>0$ into play) so that, after these
changes have taken place, each $c^* \in C$ will be a unique maximizer 
of $P^{O,\ol{R}}(M(c^*|c))$. This would imply that $M$ can be viewed
as a teacher for $L$, which would complete the proof. Details follow. \\
We enter Stage $1$ of the parameter construction. Let $L$ be the
MLE-learner whose parameters are given by 
\[
P(z|c) = \left\{ \begin{array}{ll}
    2^{-i} & \mbox{if $1 \le i \le d(c)$ and $z = z_i^c$} \\
    \frac{2^{-d(c)}}{m-d(c)} & \mbox{if $d(c)+1 \le i \le m$ and $z = z_i^c$} \\
    0 & \mbox{if $z \in Z \sm Z_c$}
        \end{array} \right. \enspace .
\]
In other words, given $c$, $L$ assigns probability mass $2^{-i}$ 
to the $i$-the element of the sequence $M(c)$ and distributes 
the remaining probability mass, $2^{-d(c)}$, evenly on the elements
of $Z_c \sm M(c)$. Note that the $c$-conditional likelihood of an 
element in $M(c)$ is at least $2^{-d(c)}$ while the probability of an
element in $Z_c \sm M(c)$ equals $\frac{2^{-d(c)}}{m-d(c)} \le 2^{-d(c)}$
with equality only if $d(c) = m-1$. It is easy to determine the
$c$-conditional likelihood of $M(c)$:
\[
P^{O,\ol{R}}(M(c) | c) =
\frac{\prod_{i=1}^{d(c)}2^{-i}}{\prod_{i=1}^{d(c)-1}2^{-i}} = 2^{-d(c)}
\enspace .
\]
The middle term contains in the numerator the product of the 
$c$-conditional likelihoods of $z_1^c,\ldots,z_{d(c)}^c$,
respectively. In the denominator, it contains the product
of the corresponding normalization factors: if $z_1^c,\ldots,z_j^c$
haven been sampled within the first $j$ trials,
then the remaining probability mass 
equals $1-\sum_{i=1}^{j}2^{-i} = 2^{-j}$.
Let us now fix an arbitrary target concept $c^* \in C$ and see
how the $c^*$-conditional likelihood of $M(c^*)$ relates to the
$c$-conditional likelihood of $M(c^*)$ for some other 
concept $c \in C \sm \{c^*\}$. We aim at showing
that $P^{O,\ol{R}}(M(c^*) | c) \le P^{O,\ol{R}}(M(c^*) | c^*)$.
We may assume that $c \models M(c^*)$ because, otherwise, we 
would obtain $P^{O,\bar R}(M(c^*)) = 0$, and we were done. 
For sake of simplicity, we set $d := d(c^*)$ and $z_i := z_i^{c^*}$ 
for $i=1,\ldots,d$. \\
Let us briefly discuss the case that $M(c)$ and $M(c^*)$
are equal as sets. Then there exists a permutation $\sigma$
such that $M(c) = z_{\sigma(1)},\ldots,z_{\sigma(d)}$.
Since $M$ is a matching, $\sigma$ cannot be the identity permutation.
It follows that $P^{O,\ol{R}}(M(c^*)|c^*) > P^{O,\ol{R}}(M(c^*)|c)$
because $(P(z_i|c^*))_{i=1,\ldots,d} = (2^{-i})_{i=1,\ldots,d}$ is 
a strictly decreasing sequence while $(P(z_i|c))_{i=1,\ldots,d}$
(as a non-identity permutation of $(2^{-i})_{i=1,\ldots,d}$) is 
not.\footnote{Compare with Remark~\ref{rem:identity-permutation}.}  \\
From now, we assume that $M(c)$ and $M(c^*)$ are different even
when viewed as sets. Let $j$ be the number of $ z \in Z$ occurring 
in $M(c)$ and in $M(c^*)$.
We can make the pessimistic assumption that the sequences $M(c)$
starts with $z_1,\ldots,z_j$ because this will lead to the largest 
conceivable value of $P^{O,\ol{R}}(M(c^*) | c)$.\footnote{This
brings the $j$ largest $c$-conditional likelihoods into play
and puts them in the most effective position.} The remaining
observations $z_{j+1},\ldots,z_{d(c)}$ must then be members 
of $Z_c \sm M(c)$. Remember that for each $z \in Z_c \sm M(c)$ 
we have that $P(z|c) = \frac{2^{-d(c)}}{m-d(c)}$. 
The term $P^{O,\ol{R}}(M(c^*) | c)$ can be expressed as a product
of two terms. The first one (resp.~second one) is the contribution
of the first $j$ trials (resp.~the last $d-j$ trials).
Since $M(c)$ starts with $z_1,\ldots,z_j$, the first term is
simply $T_1 := 2^{-j}$. The second term has the following form
\[  
T_2 := \frac{ \left( \frac{2^{-j}}{m-j} \right)^{d-j} }
{ 2^{-j} \left( 2^{-j}-\frac{2^{-j}}{m-j} \right)
\left( 2^{-j}-2\frac{2^{-j}}{m-j} \right) \ldots
\left( 2^{-j}-(d-j-1)\frac{2^{-j}}{m-j} \right) } \enspace .
\]
As usual, the numerator contains the product of the $c$-conditional
(here: uniform) likelihoods while the denominator contains the
product of the corresponding normalization factors. $T_2$ looks
terrifying at first glance, but luckily there is a lot of cancellation
and $T_2$ can be rewritten as follows:
\begin{eqnarray*}
T_2 & = & \frac{1}{(m-j)^{d-j} \left( 1-\frac{1}{m-j} \right)
\left( 1-\frac{2}{m-j} \right) \ldots
\left( 1-\frac{d-j-1}{m-j} \right) } \\
& = & \frac{1}{(m-j)(m-j-1)(m-j-2) \ldots (m-d+1)} \enspace .
\end{eqnarray*}
Remember that $d = d(c^*) \le m-1$. It follows 
that $m-d+1 \ge 2$ and therefore
\[
T_2 \le 2^{-(d-j)}\ \mbox{ and }\ P^{O,\ol{R}}(M(c^*|c) = 
T_1 \cdot T_2 \le 2^{-d}   
\]
with equality only if either $j=d$ or $d = m-1$ and $j = m-2$. 
Note that $j=d$ if and only if the sequence $M(c)$ starts with
the sequence $M(c^*) = z_1,\ldots,z_d$. \\
We enter now Stage 2 of the parameter construction, in which we
make some infinitesimal changes of the parameters that we have used
so far. In order to distinguish the new parameter collection from
the old one, the new parameters are denoted by $P_\ve(z|c)$.
They are defined as follows:
\[
P_\ve(z|c) = \left\{ \begin{array}{ll}
    2^{-i} & \mbox{if $1 \le i \le d(c)-1$ and $z = z_i^c$} \\
    2^{-i}+\ve & \mbox{if $i = d(c)$ and $z = z_i^c$} \\
    \frac{2^{-d(c)}-\ve}{m-d(c)} & \mbox{if $d(c)+1 \le i \le m$ and $z = z_i^c$} \\
    0 & \mbox{if $z \in Z \sm Z_c$}
        \end{array} \right. \enspace .
\]
The main difference to the old parameter collection is the ``extra-bonus'' $\ve$
that $c$ assigns to the last element $z_{d(c)}^c$ of the sequence $M(c)$.
Now the total probability mass assigned to $z_1^c,\ldots,z_{d(c)}^c$
is by the amount of $\ve$ greater than before, so that only probability
mass $2^{-d(c)}-\ve$ is left for $Z_c \sm M(c)$. Again, this probability
mass is shared evenly among the elements of $Z_c \sm M(c)$. Here comes
the central observation:
\begin{description}
\item[Claim:] If $\ve > 0$ is sufficiently small, then the following
implications are valid:
\begin{eqnarray*}
P^{O,\ol{R}}(M(c^*)|c^*) > P^{O,\ol{R}}(M(c^*)|c) & \impl &
P_\ve^{O,\ol{R}}(M(c^*)|c^*) > P_\ve^{O,\ol{R}}(M(c^*)|c) \enspace , \\
P^{O,\ol{R}}(M(c^*)|c^*) = P^{O,\ol{R}}(M(c^*)|c) & \impl &
P_\ve^{O,\ol{R}}(M(c^*)|c^*) > P_\ve^{O,\ol{R}}(M(c^*)|c) \enspace .
\end{eqnarray*}
\item[Proof of the Claim:]
The first implication is based on a simple continuity argument.
The second implication can be verified as follows. Remember from the
discussion in Stage 1 that $P^{O,\ol{R}}(M(c^*)|c^*) = P^{O,\ol{R}}(M(c^*)|c)$
can occur only if either $M(c)$ starts with $M(c^*) = z_1,\ldots,z_d$
or if $d = m-1$ and $j = m-2$. In the former case, the effect 
of $P_\ve(z_d|c^*) = P(z_d|c^*) + \ve$ and $P_\ve(z_d|c) = P(z_d|c)$ 
will be that 
\begin{equation} \label{eq:eps-bonus}
P_\ve^{O,\ol{R}}(M(c^*)|c^*) > P^{O,\ol{R}}(M(c^*)|c^*) =
P^{O,\ol{R}}(M(c^*)|c) = P_\ve^{O,\ol{R}}(M(c^*)|c) \enspace ,
\end{equation}
as desired. In the latter case, we have $M(c^*) = z_1,\ldots,z_{m-1}$
and either $M(c) = z_1,\ldots,z_{m-2}$ or $M(c) = z_1,\ldots,z_{m-2},z_m$.
In the latter case, we again end up at~(\ref{eq:eps-bonus}).
Suppose therefore that  $M(c^*) = z_1,\ldots,z_{m-1}$
and $M(c) = z_1,\ldots,z_{m-2}$. Here the situation is less clear, 
because the $\ve$-bonus will affect not only the $c^*$-conditional
likelihood of $M(c^*)$ but also the $c$-conditional likelihood.
We therefore compute both quantities and compare them afterwards.
Clearly $P_\ve^{O,\ol{R}}(M(c^*)|c^*) = 2^{-(m-1)} + \ve$.
The term $P_\ve^{O,\ol{R}}(M(c^*) | c)$ can be expressed as a product
of two terms, The first one (resp.~second one) is the contribution
of the first $m-2$ trials (resp.~the last trial).
Since $M(c) = z_1,\ldots,z_{m-2}$, the first term clearly
equals $2^{-(m-2)} + \ve$. Note that $2^{-(m-2)}-\ve$ is the
probability mass remaining for, and evenly shared by, $z_{m-1}$
and $z_m$. The second term equals therefore
\[
\frac{P_\ve(z_{m-1}|c)}{2^{-(m-2)}-\ve} =
\frac{\left({2^{-(m-2)}-\ve}\right) / 2}{2^{-(m-2)}-\ve} = 
\frac{1}{2} \enspace .
\]
It follows that
\[ 
P_\ve^{O,\ol{R}}(M(c^*) | c) = \frac{1}{2} \cdot \left(2^{-(m-2)} + \ve\right) = 
2^{-(m-1)} + \frac{\ve}{2} \enspace ,
\]
which is less than $P_\ve^{O,\ol{R}}(M(c^*)|c^*) = 2^{-(m-1)} + \ve$.
This completes the proof of the claim.
\end{description}
The above discussions show that we can view $M$ a teacher for the 
learner $L$ with parameter collection $(P_\ve(z|c))_{z \in Z,c \in C}$.
This completes the proof of the theorem.
\end{proof}

Combining Theorems~\ref{th:maptd-smn-set} and~\ref{th:maptd-smn-repfree}
with what we already know about saturating matching numbers, we obtain
the following result:

\begin{corollary} \label{cor:map-hierarchy}
Suppose that $(C,Z,\models)$ is of the form as described 
in Example~\ref{ex:labeled-examples} and $(\alpha,\beta) \neq (O,R)$.  
Then 
\[ 
\MAPTD^{\alpha,\beta}(C) = \SMN(G(C)^{\alpha,\beta})\ \mbox{ and }\ 
\MLETD^{\alpha,\beta}(C) = \SMN(G(C)_{\neq\eset}^{\alpha,\beta}) 
\enspace .
\]
Moreover
\begin{eqnarray*}
\MAPTD^{\ol{O},\ol{R}}(C) & \ge & \max\{\MAPTD^{O,R}(C) , \MAPTD^{O,\ol{R}}(C)\}
\enspace , \\
\MLETD^{\ol{O},\ol{R}}(C) & \ge & 
\max\{\MLETD^{O,R}(C) , \MLETD^{O,\ol{R}}(C)\} \enspace .
\end{eqnarray*}
\end{corollary}

The first assertion of the corollary implies the correctness of 
the results which are visualized in Fig.~\ref{fig:smn-characterization}.
The following two results provide some supplementary information:

\begin{theorem} \label{th:strict-smn-hierarchy}
Let $(\alpha,\beta)$ and $(\alpha',\beta')$ be two different
sampling modes. There exists a concept class $C$ such
that $\SMN(G(C)^{\alpha',\beta'}) \neq \SMN(G(C)^{\alpha,\beta})$.
\end{theorem}

\begin{proof}
We present the proof for $(\alpha,\beta) = (\ol{O},R)$
and $(\alpha',\beta') = (\ol{O},\ol{R})$.\footnote{The proof for the
other choices of $(\alpha,\beta)$ and $(\alpha',\beta')$ is similar.}
Let $X = \{x_1,\ldots,x_m\}$, let $Z = X \times \{0,1\}$,
let $C_m$ be the powerset of $X$ and let $\models$ be given
by~(\ref{eq:labeled-examples}). Let $\cZ_2$ (resp.~$\cZ'_2$)
be the set of all $A \in \cZ^{\ol{O},R}$ (resp.~$A \in \cZ^{\ol{O},\ol{R}}$)
such that $|A| \le 2$. A simple counting argument shows
that $|\cZ'_2| < |\cZ_2|$. Consider the bipartite graph $G$
with vertex sets $C_m$ and $\cZ_2$ and with an edge $(c,A)$ if
and only if $c \models A$. Each vertex in $\cZ_2$ has degree 
at least $D := 2^{m-2}$ whereas each vertex in $C_m$ has 
degree $d := 1+2m+\frac{1}{2}(m-1)m$. Suppose that $m$ is sufficiently 
large such that $d \le D$. Fix an arbitrary subset $S$ of $\cZ_2$. 
It follows that
\[ |\Gamma(S)| \ge \frac{D}{d} \cdot |S| \ge |S| \]
so that $G$ satisfies Hall's condition. It follows that $G$
admits a $\cZ_2$-saturating matching, say $M$. 
Let $C$ be the set of concepts in $C_m$ having an $M$-partner.
By construction: $\SMN(G(C)^{\ol{O},R}) = 2$.
For cardinality reasons, namely $|C| = |M| = |\cZ_2| > |\cZ'_2|$,
we have $\SMN(G(C)^{\ol{O},\ol{R}}) > 2$. 
\end{proof}

Theorem~\ref{th:strict-smn-hierarchy} implies that the parameters
with different colors in Fig.~\ref{fig:smn-characterization} can
generally have different values.

\subsection{MAP- versus MLE-Learners} \label{subsec:map-versus-mle}

Suppose that $L$ is an MLE-learner for $C$. Let $L'$ be a MAP-learner
that differs from $L$ only by having non-uniform priors, i.e., the
conditional likelihoods are the same. The following example
demonstrates that the gap between $\MAPTD_L^{\alpha,\beta}(C)$ 
and $\MAPTD_{L'}^{\alpha,\beta}(C)$ can become arbitrarily 
large.\footnote{This example uses a concept class, namely singletons 
plus empty set, which is often used to demonstrate that the classical 
teaching model from~\cite{SM1991,GK1995} may assign an inappropriately high 
teaching dimension to a trivial concept class.}

\begin{example} \label{ex:map-mle-gap}
Let $X = \{x_1,\ldots,x_m\}$, $Z = X \times \{0,1\}$,
$C = \{\{x_1\},\ldots,\{x_m\}\} \cup \{\eset\}$ and let $\models$
be given by~(\ref{eq:labeled-examples}). Consider the 
MLE-learner $L$ be given by the parameters 
\[ P((x_i,c(x_i)) | c)\ =\ \frac{1}{m} \]
for each $c \in C$ and $i=1,\ldots,m$. We assume for simplicity
that the sampling mode $(\alpha,\beta)$ of $L$ equals $(\ol{O},\ol{R})$, 
but the following reasoning (mutatis mutandis) applies to any other 
sampling mode as well. Clearly, for each $k \in [m]$, the concept $\{x_k\}$ 
can be taught by the single observation $(x_k,1)$. However $\eset$ can only 
be taught by the full set $A_0 := \{(x_i,0): i=1,\ldots,m\}$ of 
observations that $\eset$ is consistent with: as long as some $(x_k,0)$ 
is missing in a set $A \subset A_0$, we have that $P(A|\eset) = P(A|\{x_k\})$
so that $\eset$ is not the unique maximizer of $P(A|c)$.
We may conclude from this discussion 
that $\MAPTD_L^{\alpha,\beta}(C) = m$. Let $L'$ be a MAP-learner
that differs from $L$ only by having for $\eset$ a higher prior
than for the other concepts in $C$. Then the concept $\{x_k\}$
can still be taught by the single observation $(x_k,1)$. But now
also the concept $\eset \in C$ can be taught in a trivial fashion 
by $\eset \in 2^Z$. We may conclude 
that $\MAPTD_{L'}^{\alpha,\beta}(C) = 1$.
\end{example}

In contrast to Example~\ref{ex:map-mle-gap}, the next result shows
that, in case of optimally parameterized learners, the advantage of 
MAP-learners over MLE-learners is all but dramatic:

\begin{theorem}
Suppose that $(C,Z,\models)$ is of the form as described
in Example~\ref{ex:labeled-examples} and $(\alpha,\beta) \neq (O,R)$.
Then
\begin{equation} \label{eq1:map-mle}
\MAPTD^{\alpha,\beta}(C) \le \MLETD^{\alpha,\beta}(C) \le 
1+\MAPTD^{\alpha,\beta}(C) \enspace .
\end{equation} 
Moreover, there exist concept classes $C'$ and $C''$ such that 
\begin{equation} \label{eq2:map-mle}
\MLETD^{\alpha,\beta}(C') = \MAPTD^{\alpha,\beta}(C')\  
\mbox{ and }\ \MLETD^{\alpha,\beta}(C'') =  1+\MAPTD^{\alpha,\beta}(C'')
\enspace .
\end{equation}
\end{theorem}

\begin{proof}
Clearly $\MAPTD^{\alpha,\beta}(C) \le \MLETD^{\alpha,\beta}(C)$.
In order to obtain~(\ref{eq1:map-mle}), it suffices therefore to 
show that $\MLETD^{\alpha,\beta}(C) \le 1 + \MAPTD^{\alpha,\beta}(C)$,
or equivalently, that 
$\SMN(G(C)_{\neq\eset}^{\alpha,\beta}) \le 1 + \SMN(G(C)^{\alpha,\beta})$.
We present the proof for $(\alpha,\beta) = (\ol{O},\ol{R})$.\footnote{The
proof for the other choices of $(\alpha,\beta)$ is similar.} 
For sake of brevity, set $m := |X|$, $G = G(C)^{\ol{O},\ol{R}}$
and $d := \SMN(G)$. Since $\SMN(G_{\neq\eset}) \le m$, 
we may assume that $d \le m-2$. Let $M:C \ra 2^Z$ be a $C$-saturating 
matching of order $d$ in $G$. If $M$ does not assign $\eset$ to any 
concept in $C$, then $\SMN(G_{\neq\eset}) \le d$. Otherwise, 
if $M(c_0) = \eset$ for some $c_0 \in C$, then we may arbitrarily 
pick a set $A \subset X$ of size $d+1$ and replace the $M$-partner $\eset$ 
of $c_0$ by the set $B = \{(a,c_0(a)): a \in A\}$. The resulting matching 
now witnesses that $\SMN(G_{\neq\eset}) \le d+1$. \\
We still have to specify concept classes $C'$ and $C''$
which satisfy~(\ref{eq2:map-mle}). As for $C'$, there are plenty 
of choices, e.g., $C' = \{ \{x_i\}: i = 1,\ldots,m\}$
satisfies 
\[ 
\MLETD^{\alpha,\beta}(C') = \MAPTD^{\alpha,\beta}(C') = 1 \enspace .
\]
In order to specify an appropriate class $C''$, we assume again
that $(\alpha,\beta) = (\ol{O},\ol{R})$ and proceed as follows.
Let $X = \{x_1,\ldots,x_m\}$, let $Z = X \times \{0,1\}$,
let $C_m$ be the powerset of $X_m$ and let $\models$ be given
by~(\ref{eq:labeled-examples}). Let $\cZ_{\le d}$ 
(resp.~$\cZ'_{\le d}$) be the set of subsets (resp.~non-empty subsets)
of $Z$ of size at most $d$. Consider the bipartite graph $G$ with vertex 
sets $C_m$ and $\cZ_{\le d}$ an edge $(c,A)$ if and only if $c \models A$. 
If $m$ is sufficiently large (while $d$ is kept fixed), $G$ admits 
a $\cZ_{\le d}$-saturating matching, say $M$.  Let $C''$ be the set of 
concepts in $C_m$ having an $M$-partner. 
By construction: $\SMN(G(C'')^{\ol{O},\ol{R}}) = d$.
For cardinality reasons, 
namely $|C''| = |M| = |\cZ_{\le d}| > |\cZ_{\le d}| - 1 = |\cZ'_{\le d}|$,
we have $\SMN(G(C'')_{\neq\eset}^{\ol{O},\ol{R}}) > d$, which implies
that $\SMN(G(C'')_{\neq\eset}^{\ol{O},\ol{R}}) = d+1$.
\end{proof}

\subsection{Parameters Bounding MLE-TD from Above} 
\label{subsec:mletd-upper-bounds}

Since $\MLETD$ can never be smaller than $\MAPTD$, it follows
that $\MLETD^{\ol{O},\ol{R}}(C)$ is the largest among the
parameters occurring in Corollary~\ref{cor:map-hierarchy}.
Hence upper bounds on $\MLETD^{\ol{O},\ol{R}}(C)$ are, all the more,
upper bounds on the other parameters. For this reason, we confine 
ourselves to MLE-learners and to sampling mode $(\ol{O},\ol{R})$
in what follows. In order to simplify notation, we will write 
\begin{itemize}
\item
$2^Z$ instead of $\cZ^{\ol{O},\ol{R}}$, 
\item 
$\MLETD(C)$ instead of $\MLETD^{\ol{O},\ol{R}}(C)$, 
\item
$G^+(C)$ instead of $G(C)^{\ol{O},\ol{R}}_{\neq\eset}$. 
\end{itemize}

Among the parameters that bound $\MLETD(C)$ from above are the
antichain number of $C$, the VC-dimension of $C$ and the so-called 
no-clash teaching dimension of $C$. We begin with the definition 
of the antichain number:

\begin{definition}[Antichain Mapping and Antichain Number]
\label{def:antichain}
$T:C \ra 2^Z$ is called an \emph{antichain mapping for $C$} 
if the following holds:
\begin{enumerate}
\item
Each concept $c \in C$ is consistent with $T(c)$.
\item
The sets $(T(c))_{c \in C}$ form an antichain, i.e.,
\[ 
\forall c_1 \neq c_2 \in C: 
T(c_1) \not\seq T(c_2) \wedge T(c_2) \not\seq T(c_1) \enspace . 
\]
\end{enumerate}
The smallest possible  order of an antichain mapping for $C$ 
is called the \emph{antichain number of $C$} and denoted by $\AN(C)$.
\end{definition}

\noindent
It is well-known that the antichain number is upper-bounded by 
the VC-dimension:

\begin{theorem}[\cite{MSSZ2022}] \label{th:an-vcd}
Suppose that the concept class $C$ is a family of subsets 
of a finite domain $X$. Then $\AN(C) \leq \VCD(C)$. 
\end{theorem}

\noindent
We proceed with the definition of the teaching dimension 
in the so-called no-clash model of teaching:

\begin{definition}[No-clash Teaching Dimension~\cite{KSZ2019,FKSSZ2022}]
\label{def:nctd}
A mapping $T: C \ra 2^Z$ is called \emph{clash-free on $C$} 
if it satisfies the following:
\begin{enumerate}
\item Each $c \in C$ is consistent with $T(c)$.
\item
If $c_1 \neq c_2 \in C$, then $c_1$ is inconsistent with $T(c_2)$
or $c_2$ is inconsistent with $T(c_1)$.\footnote{The situation
that $c_1$ is consistent with $T(c_2)$ and $c_2$ is consistent
with $T(c_1)$ would be called a \emph{clash of $c_1$ and $c_2$}.
This explains why the mapping $T$ is called clash-free.}
\end{enumerate}
The \emph{no-clash teaching dimension of $C$}, denoted
as $\NCTD(C)$,  is the smallest possible order of a 
mapping $T:C \ra 2^Z$ that is clash-free on $C$. 
\end{definition}

\begin{theorem} \label{th:mle-others}
Suppose that $(C,Z,\models)$ is of the form as described 
in Example~\ref{ex:labeled-examples}. Then $\MLETD(C) \le \AN(C)$ 
and $\MLETD(C) \le \NCTD(C)$.
\end{theorem}

\begin{proof}
Because $\MLETD(C) = \SMN(G^+(C))$, 
it suffices to show that $\SMN(G^+(C))$ is upper-bounded by $\AN(C)$ and $\NCTD(C)$.
An antichain mapping $T:C \ra 2^Z$ clearly satisfies~(\ref{eq:injectivity})
and does not have $\eset$ in its image. Thus, an application of
Remark~\ref{rem:smn-lb} yields $\AN(C) \ge \SMN(G^+(C))$.
A clash-free mapping $T:C \ra 2^Z$ must be of order at least~$1$. 
There can be at most one concept $c$ in $C$ 
such that $T(c) = \eset$. Suppose that $T(c)=\eset$. Consider
an arbitrary, but fixed, concept $c' \in C\sm\{c\}$.  Since $c'$ 
is consistent with (the empty sample) $T(c)$ and $T$ is clash-free, 
the concept $c$ must be inconsistent with $T(c')$. Let us 
redefine $T(c)$ as a singleton set $\{(x,b)\}$ such that $b = c(x)$. 
This modification of $T$ is still clash-free and leaves the order
of $T$ unchanged. Moreover, after this modification, 
$T$ satisfies ~(\ref{eq:injectivity}) and does not have $\eset$ in its
image. Now another application of Remark~\ref{rem:smn-lb} 
yields $\NCTD(C) \ge \SMN(G^+(C))$. 
\end{proof}
The inequality $\MLETD(C) \le \NCTD(C)$ had been proven already 
in~\cite{IJCAI}. The proof given there does not make use
of saturating matching numbers and is more complicated. 
Because $\AN(C) \le \VCD(C)$, we immediately obtain the following 
result:

\begin{corollary} \label{cor:mle-lb-vcd}
Suppose that $(C,Z,\models)$ is of the form as described
in Example~\ref{ex:labeled-examples}. Then $\MLETD(C) \le \VCD(C)$.
\end{corollary}

\subsection{Computational Considerations} \label{subsec:polytime}

We will show in the course of this section that $\SMN(G^+(C))$
(and related quantities) can be computed in time poly$(|C|,|X|)$
from a given (finite) concept class $C \seq 2^X$. The central 
observation will be that, in order to find a $C$-saturating 
matching of minimum order in $G^+(C)$, we do not need to compute 
the (possibly exponentially large) bipartite graph $G^+(C)$.
All pieces of information about $G^+(C)$ that we need
in the course of the algorithm can be efficiently extracted 
from the much smaller bipartite graph $G(C)$. 

We start with a lemma that is particularly interesting when we have 
a bipartite graph whose first vertex set, $V_1$, is much smaller 
than its second vertex set, $V_2$:

\begin{lemma} \label{lem:oracle-algorithm}
Let $G = (V_1,V_2,E)$ with $E \seq V_1 \times V_2$ be a bipartite graph.
Let $\cO$ be an oracle that, upon request $(v,k)$ with $v \in V_1$
and $k \in [|V_1|]$, returns $\min\{\deg_G(v),k\}$ distinct neighbors 
of $v$.\footnote{The oracle $\cO$ can be implemented efficiently if,
for instance, $G$ is represented by the adjacency lists for the vertices
in $V_1$ and there is direct access to each of these lists.}
Then there is an oracle algorithm $A^\cO$ which computes a maximum 
matching in $G$ and has a time bound that is polynomial in~$|V_1|$.
\end{lemma}

\begin{proof}
For sake of brevity, we set $n = |V_1|$. Let $V'_1 \seq V_1$ be the 
set of vertices in $V_1$ with less than $n$ neighbors, and
let  $V''_1 = V_1 \sm V'_1$ be the set of remaining vertices in $V_1$,
i.e., the vertices with at least $n$ neighbors. The algorithm $A^\cO$
proceeds as follows:
\begin{enumerate}
\item 
For each $v \in V_1$, it sends the request $(v,n)$ to $\cO$ and
receives a list of all neighbors if $v \in V'_1$, resp.~a list 
of $n$ distinct neighbors if $v \in V''_1$. 
\item
Now $A^\cO$ computes a maximum matching $M'$ in the subgraph $G'$ of $G$ 
that is induced by $V'_1$ and $\Gamma(V'_1)$. 
\item
$A^\cO$ augments $M'$ to a $V_1$-saturating matching in a greedy fashion:
for each $v  \in V''_1$, it inspects the list of $n$ distinct neighbors
of $v$ and matches $v$ with the first neighbor which had not been matched
before.
\end{enumerate}
Note that $G'$ has at most $n(n-1)$ vertices. Moreover, among $n$ neighbors 
of a vertex $v \in V''_1$, there must be at least $1$ neighbor which is not 
already matched with another vertex in $V_1$. It easily follows that $A^\cO$ 
returns a maximum matching in poly$(|V_1|)$ time.
\end{proof}

\noindent
With a bipartite graph $G = (V_1,V_2,E)$, we associate 
the bipartite graph 
\begin{equation} \label{eq:extended-graph}
G^+ = (V_1,2^{V_2}\sm\{\eset\},E^+)\ \mbox{ with }\ 
E^+ = \{(v,B) \in V_1 \times 2^{V_2}\sm\{\eset\}: \{v\} \times B \seq E\}
\enspace .
\end{equation}
In other words: the pair $(v,B)$ with $v \in V_1$ 
and $\eset \subset B \seq V_2$ is an edge in $E^+$ iff, 
for every $v' \in B$, the pair $(v,v')$ is an edge in $E$.

\begin{theorem} \label{th:matching-algorithm}
Given a bipartite graph $G = (V_1,V_2,E)$, a $V_1$-saturating matching 
of minimum order in $G^+$ (resp.~an error message if a $V_1$-saturating
matching does not exist) can be computed in polynomial time:
\end{theorem}

\begin{proof}
We consider first the problem of computing a $V_1$-saturating 
matching of minimum order in $G^+$. Let us fix some notation. 
For $\ell=1,\ldots,|V_2|$, 
let $G^{(\ell)} = (V_1,V_2^{(\ell)},E^{(\ell)})$ be the bipartite graph 
given by
\[
V_2^{(\ell)} = \{B \seq V_2: 1 \le |B| \le \ell\})\ \mbox{ and }\
E_2^{(\ell)} = \{(v,B) \in V_1 \times V_2^{(\ell)}: \{v\} \times B \seq E\}
\enspace .
\]
In other words, $G^{(\ell)}$ is the subgraph of $G^+$ induced by $V_1$
and $V_2^{(\ell)}$. Given $G$, $\ell \in [|V_2|]$, $k \in [|V_1|]$
and $v \in V_1$, it is 
easy to compute a list of $\min\{\deg(v),k\}$ distinct neighbors
of $v$ in $G^{(\ell)}$. It follows from Lemma~\ref{lem:oracle-algorithm}
that, given $G$ and $\ell \in [|V_2|]$, we can compute in poly$(|V_1|,|V_2|)$
steps a maximum matching $M_{\ell}$ in $G^{(\ell)}$. Let $\ell^+$
be the minimum $\ell$ such that $M_\ell$ is of size $|V_1|$,
respectively $\ell^+ = 1+|V_2|$ if none of the $M_\ell$ 
saturates $V_1$. If $\ell^+ \le |V_2|$, then $M_{\ell^+}$
is the desired $V_1$-saturating matching of minimum order in $G^+$.
If $\ell^+ = |V_2|+1$, we may report error because $G^+$ does not
admit a $V_1$-saturating matching. 
\end{proof}

\begin{corollary} \label{cor:poly-smn}
Suppose that $(C,Z,\models)$ is of the form as described
in Example~\ref{ex:labeled-examples}. Then the following 
objects can be computed in polynomial time:
\begin{itemize}
\item
the bipartite consistency graph $G(C)$ with vertex sets $C$ and $Z$
\item 
the (identical) parameters $\SMN(G^+(C))$ and $\MLETD(C)$
\item 
a $C$-saturating matching $M$ in $G^+(C)$ of order $\SMN(G^+(C))$ 
\item
parameters representing an MLE-learner $L$ for $C$ and a teacher $T$
for $L$ who is of order $\MLETD(C)$
\end{itemize}
\end{corollary}

\begin{proof}
Given $C$, the set $Z$ and the bipartite graph $G(C)$ can clearly
be computed in polynomial time. We may now apply 
Theorem~\ref{th:matching-algorithm} to the bipartite graph $G = G(C)$. 
Then $G^+$ in Theorem~\ref{th:matching-algorithm} equals $G^+(C)$, 
Hence the algorithm sketched in the proof 
of Theorem~\ref{th:matching-algorithm} can be used for finding 
a $C$-saturating matching $M$ in $G^+(C)$ of minimum order (which 
is order $\SMN(G^+(C))$). As a byproduct, the parameter $\SMN(G^+(C))$ 
is now known as well. 
As for the specification of an appropriate MLE-learner $L$,
we may use the parameter setting that is found in the proof 
of Theorem~\ref{th:smn-bounds}. As also shown 
in that proof, $M$ (already known to be computable from $C$ in
polynomial time) represents a teacher of order $\MLETD(C)$ for $L$.
This completes the proof of the corollary.
\end{proof}

It is straightforward to extend Corollary~\ref{cor:poly-smn}
from sampling mode $(\ol{O},\ol{R})$ to other sampling modes,
and from $\MLETD$ to $\MAPTD$. The main point is to adjust
the definition of $G^+$ in~(\ref{eq:extended-graph}) 
so that $G(C)^+$ becomes identical to $G(C)_{\neq\eset}^{\alpha,\beta}$
resp.~to $G(C)^{\alpha,\beta}$. We omit the details.

\paragraph{Open Problems and Future Work.}

What are ``natural parameterizations'' of MAP- or MLE-learners?
Does MAP-based teaching of naturally parameterized learners force
the teacher to present observations/examples which illustrate
the underlying target concept in an intuitively appealing way?

\appendix

\section{Proof of Facts 1--4}

\begin{description}
\item[Fact 1:]
Suppose that $0 \le |\Delta| < p < 1/2$.
Let $c^{\pm\Delta}$ be the concept given by~(\ref{eq:plus-minus-Delta}).
Then $P^{O,R}(z_1,z_2)|c^{\pm\Delta})$
and $P^{\ol{O},\ol{R}}(z_1,z_2|c^{\pm\Delta})$ are both
strictly decreasing when $|\Delta|$ is increased.
\end{description}

\begin{proof}
The assertion is obvious 
for $P^{O,R}(z_1,z_2)|c^{\pm\Delta}) = (p+\Delta)(p-\Delta) = p^2-\Delta^2$.
Consider now the function
\[ 
h(\Delta) := P^{\ol{O},\ol{R}}(z_1,z_2|c^{\pm\Delta}) =
\frac{(p+\Delta)(p-\Delta)}{1-p-\Delta} +  \frac{(p-\Delta)(p+\Delta)}{1-p+\Delta} =
\frac{2(1-p)(p^2-\Delta^2)}{(1-p)^2-\Delta^2} \enspace ,
\]
where the last equation can be obtained by a straightforward calculation.
Another straightforward, but tedious, calculation shows that
\[
h'(\Delta) = - \frac{4(1-p)(1-2p)\Delta}{((1-p)^2-\Delta^2)^2}.
\]
Hence the function $h(\Delta)$ is strictly increasing for $\Delta < 0$
and strictly decreasing for $\Delta > 0$. It is therefore strictly
decreasing when $|\Delta|$ is increased.
\end{proof}

\begin{description}
\item[Fact 2:]
Suppose that $0 \le \Delta < p < 1/2$. Let $c^{\pm\Delta}$ be
the concept given by~(\ref{eq:plus-minus-Delta}). Then
\[
P^{O,\ol{R}}(z_1,z_2 | c^{\pm\Delta}) - P^{O,\ol{R}}(z_1,z_2 | c^{\pm0})
\left\{ \begin{array}{ll}
   = 0 & \mbox{if $\Delta \in \{0,\frac{p^2}{1-p}\}$} \\
   > 0 & \mbox{if $0 < \Delta < \frac{p^2}{1-p}$} \\
   < 0 & \mbox{otherwise}
\end{array} \right. \enspace .
\]
\end{description}

\begin{proof}
We set
\[ 
h(\Delta) := P^{O,\ol{R}}(z_1,z_2|c^{\pm\Delta}) =
\frac{(p+\Delta)(p-\Delta)}{1-p-\Delta} = \frac{p^2-\Delta^2}{1-p-\Delta}
\]
and observe that
\begin{eqnarray*}
P^{O,\ol{R}}(z_1,z_2 | c^{\pm\Delta}) - P^{O,\ol{R}}(z_1,z_2 | c^{\pm0}) 
& = & 
h(\Delta) - h(0) \\
& = &
\frac{(1-p)(p^2-\Delta^2) - (1-p-\Delta)p^2}{(1-p-\Delta)(1-p)} \\
& = & \frac{\Delta (p^2-(1-p)\Delta)}{(1-p-\Delta)(1-p)} \enspace .
\end{eqnarray*}
The denominator of the latter expression is strictly positive.
Moreover
\[
\Delta (p^2-(1-p)\Delta) \left\{ \begin{array}{ll}
   = 0 & \mbox{if $\Delta \in \{0,\frac{p^2}{1-p}\}$} \\
   > 0 & \mbox{if $0 < \Delta < \frac{p^2}{1-p}$} \\
   < 0 & \mbox{otherwise} 
\end{array} \right. \enspace ,
\]
which accomplishes the proof of Fact~2.
\end{proof}

\begin{description}
\item[Fact 3:]
Suppose that $0 \le \Delta < p < 1/2$. Let $c^{\pm\Delta}$ be
the concept given by~(\ref{eq:plus-minus-Delta}). Then
\[
P^{O,R}(z_1,z_1,z_2 | c^{\pm\Delta}) - P^{O,R}(z_1,z_1,z_2 | c^{\pm0})
\left\{ \begin{array}{ll}
   = 0 & \mbox{if $\Delta \in \{0,\frac{1}{2}\sqrt{5}-1)p\}$} \\
   > 0 & \mbox{if $0 < \Delta < \frac{1}{2}(\sqrt{5}-1)p$} \\
   < 0 & \mbox{otherwise}
\end{array} \right. \enspace .
\]
\end{description}

\begin{proof}
Let $0 < \delta < 1$ be given by $\Delta = \delta p$ and note that
\[
P^{O,R}(z_1,z_1,z_2|c^{\pm\delta p}) = (p+\delta p)^2 \cdot (p-\delta p) = 
(1+\delta)^2 \cdot (1-\delta) \cdot p^3 = (1+\delta-\delta^2-\delta^3) \cdot p^3 
\enspace .
\]
It follows that
\[
P^{O,R}(z_1,z_1,z_2|c^{\pm\delta p}) - P^{O,R}(z_1,z_1,z_2|c^{\pm0}) =
\delta \cdot (1-\delta-\delta^2) \cdot p^3 \enspace .
\]
Furthermore
\[
\delta \cdot (1-\delta-\delta^2) \left\{ \begin{array}{ll}
       = 0 & \mbox{if $\delta \in \left\{0,\frac{1}{2}(\sqrt{5}-1)\right\}$} \\
       > 0 & \mbox{if $0 < \delta < \frac{1}{2}(\sqrt{5}-1)$} \\
       < 0 & \mbox{otherwise} 
    \end{array} \right. \enspace .
\]
We may conclude from this discussion that~(\ref{eq2:odd-split}) is valid.
\end{proof}

\begin{description}
\item[Fact 4:]
Suppose that $0 < p < 1/2$ and $1 \le t < \frac{1-p}{p}$.
Let $c^{(t)}$ be the concept given by~(\ref{eq:times-t}).
Then $P^{\ol{O},\ol{R}}(z_1,z_2|c^{(t)})$ is strictly increasing
with $t$.
\end{description}

\begin{proof}
Set 
\begin{eqnarray*}
h(t) & := &\frac{P^{\ol{O},\ol{R}}(z_1,z_2|c^{(t)})}{p^2} =
\frac{1}{1-pt} + \frac{1}{1-p/t} = \frac{1}{1-pt} + \frac{t}{t-p} \\
& = &
\frac{(t-p) + (1-pt)s}{(1-pt)(t-p)} = \frac{2t-pt^2-p}{(p^2+1)t-pt^2-p}
\enspace .
\end{eqnarray*}
It suffices to show that $h(t)$ is strictly increasing with $t$. 
To this end, we compute the first derivative:
\[
h'(t) =
\frac{(2-2pt) \cdot ((p^2+1)t-pt^2-p) - (2t-pt^2-p)(p^2+1-2pt)}
{(1-pt)^2 \cdot (t-p)^2} \enspace .
\]
The denominator is strictly positive. After an application of the
distributive law and some cancellation, the numerator has the form
\[ f(t) := p(1-p^2)(t^2-1) \enspace . \]
Hence the numerator equals $0$ for $t=1$ and is strictly positive
for $t>1$. It follows that $h(t)$ with $t \ge 1$ is strictly increasing.
\end{proof}

\bibliography{map}

\end{document}